\documentclass{article}

\usepackage{amsfonts}       %
\usepackage{amsmath}  %
\usepackage{amssymb}
\usepackage{arxiv}
\usepackage{booktabs}       %
\usepackage{doi}
\usepackage{enumitem}
\usepackage{fancyvrb}   %
\usepackage[T1]{fontenc}    %
\usepackage{graphicx}
\usepackage{hyperref}
\usepackage[utf8]{inputenc} %
\usepackage{listings}
\usepackage{mathpartir}
\usepackage{mathtools}  %

\usepackage{microtype}      %
\usepackage{nicefrac}       %
\usepackage{orcidlink}
\usepackage{placeins}
\usepackage{proof}      %

\usepackage[
backend=biber,
style=alphabetic,
maxbibnames=99
]{biblatex}
\usepackage{subcaption}
\usepackage{tikz}
\usepackage{url}            %
\usepackage{xcolor}

\usepackage{cleveref}

\newcommand{\eg}{\textit{e.g.}}
\newcommand{\ie}{\textit{i.e.}}

\newcommand{\etal}{\textit{et al.}}

\DeclarePairedDelimiter\abs{\lvert}{\rvert}%
\DeclarePairedDelimiter\norm{\lVert}{\rVert}%
\DeclarePairedDelimiter{\ceil}{\lceil}{\rceil}%

\makeatletter
\let\oldabs\abs
\def\abs{\@ifstar{\oldabs}{\oldabs*}}
\let\oldnorm\norm
\def\norm{\@ifstar{\oldnorm}{\oldnorm*}}
\makeatother

\newcommand{\true}{1}
\newcommand{\false}{0}

\newcommand{\assignment}{\sigma}
\newcommand{\constraint}[1]{\mathcal{C}_{#1}}

\newcommand{\support}[1]{\textit{sup}(#1)}
\newcommand{\pbformula}{\mathcal{F}}

\newcommand{\acltwo}{{\sf ACL2}}
\newcommand{\cakepb}{{\sf CakePB}}
\newcommand{\coq}{{\sf Coq}}
\newcommand{\drat}{{\sf DRAT}}
\newcommand{\holfour}{{\sf HOL4}}
\newcommand{\isabelle}{{\sf Isabelle}}
\newcommand{\roundingsat}{{\sf RoundingSAT}}
\newcommand{\veripb}{{\sf VeriPB}}
\newcommand{\veripbvtwo}{{\sf VeriPB~v2}}

\definecolor{myblue}{HTML}{2175BA}
\definecolor{mygreen}{HTML}{137403}

\newtheorem{definition}{Definition}
\newtheorem{theorem}{Theorem}
\newtheorem{proof}{Proof}
\newtheorem{lemma}{Lemma}
\newtheorem{example}{Example}

\crefname{section}{Section}{Sections}
\crefname{definition}{Definition}{Definitions}
\crefname{figure}{Figure}{Figures}
\crefname{table}{Table}{Tables}
\crefname{theorem}{Theorem}{Theorems}
\crefname{lemma}{Lemma}{Lemmas}
\crefname{example}{Example}{Examples}

\hypersetup{
    colorlinks=true,
    linkcolor=magenta,    
    urlcolor=myblue,
    citecolor=mygreen,
    pdfauthor={Anna L.D. Latour, Arunabha Sen, Kaustav Basu, Chenyang Zhou, Kuldeep S. Meel},
    pdftitle={Latour et al. 2024 --- The Cardinality of Identifying Code Sets for Soccer Ball Graph with Application to Remote Sensing (preprint)},
    pdfpagemode=FullScreen,
    pdfkeywords={Computer Science,
    Identifying Codes,
    Vertex Identification,
    Sensor Placement,
    Network Monitoring,
    Dominating Sets,
    Automated Reasoning,
    Certification and Verification,
    Pseudo-Boolean Reasoning,
    Soccer Balls,
    Soccer Ball Graphs},
    }

\begin{document}

\title{The Cardinality of Identifying Code Sets for Soccer Ball Graph with Application to Remote Sensing}

\date{19 July 2024}	%

\author{ 
    Anna L.D. Latour\thanks{Corresponding author.}\:\,\orcidlink{0000-0002-5802-8271} \\
	School of Computing \\
	National University of Singapore\\
	Singapore, Republic of Singapore \\
	\href{mailto:a.l.d.latour@tudelft.nl}{a.l.d.latour@tudelft.nl}\\
	\And
    Arunabha Sen\thanks{Work done in part while Arunabha Sen visited the National University of Singapore.},\: Kaustav Basu, Chenyang Zhou \\
    School of Computing and Augmented Intelligence\\
    Arizona State University \\
    Tempe, Arizona, USA \\
    \href{mailto:asen@asu.edu}{asen@asu.edu}, \href{mailto:kaustavb20@gmail.com}{kaustavb20@gmail.com}, \href{mailto:czhou24@asu.edu}{czhou24@asu.edu}
    \And
    Kuldeep S. Meel\,\orcidlink{0000-0001-9423-5270} \\
    School of Computing \\
	National University of Singapore\\
	Singapore, Republic of Singapore \\
    \href{mailto:meel@cs.toronto.edu}{meel@cs.toronto.edu}
}

\maketitle

\begin{abstract}
    In the context of satellite monitoring of the earth, we can assume that the surface of the earth is divided into a set of regions. 
    We assume that the impact of a big social/environmental event spills into neighboring regions.
    Using {\em Identifying Code Sets} (ICSes), we can deploy sensors in such a way that the region in which an event takes place can be uniquely identified, even with fewer sensors than regions. 
    As Earth is almost a sphere, we use a soccer ball as a model. 
    We construct a {\em Soccer Ball Graph} (SBG), and provide human-oriented, analytical proofs that 1) the SBG has at least $26$ ICSes of cardinality ten, implying that there are {\em at least} $26$ different ways to deploy ten satellites to monitor the Earth and 2) that the cardinality of the {\em minimum} Identifying Code Set (MICS) for the SBG is {\em at least} nine. 
    We then provide a machine-oriented formal proof that the cardinality of the MICS for the SBG is in fact {\em ten}, meaning that one must deploy at least ten satellites to monitor the Earth in the SBG model. 
    We also provide machine-oriented proof that there are {\em exactly} 26 ICSes of cardinality ten for the SBG.
\end{abstract}


\section{Introduction}
\label{sec:introduction}

We study an event monitoring problem with satellites as sensors. The events that we focus on may be environmental (drought/famine), social/political (social unrest/war) or extreme events (earthquakes/tsunamis). Such events take place in {\em regions} on the surface of the earth, where a region  may be a continent, a country, or a set of neighboring countries. 

The sensors that we envisage for monitoring such events are satellites placed in orbits surrounding the earth. 
A satellite constellation that can be deployed for such monitoring purposes is shown in \cref{subfig:constellation}. 
Examples of such  constellations include the Global Positioning System (GPS) for navigation, the Iridium and Globalstar satellite telephony, and the Disaster Monitoring Constellation (DMC) for remote sensing. 
In particular, DMC is designed to provide earth imaging for disaster relief and was used extensively to monitor the impact of the Indian Ocean Tsunami in December 2004, Hurricane Katrina in August 2005, and several other floods, fires and disasters. 
The problem that we address in this paper is directly relevant to the services being provided by organizations such as the DMC. 

\begin{figure}[thb]
	\begin{center}
		\centering
        \begin{subfigure}[c]{0.27\textwidth}
            \includegraphics[width=\textwidth, keepaspectratio]{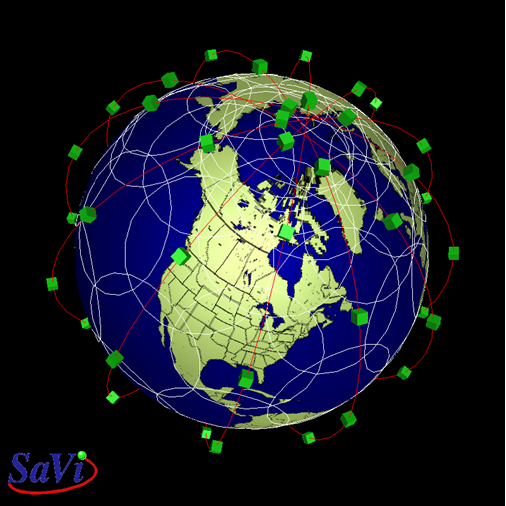}
            \caption{Satellite Constellation Covering Earth.}
            \label{subfig:constellation}
        \end{subfigure}
        \hfill
        \begin{subfigure}[c]{0.23\textwidth}
            \includegraphics[width=\textwidth, keepaspectratio]{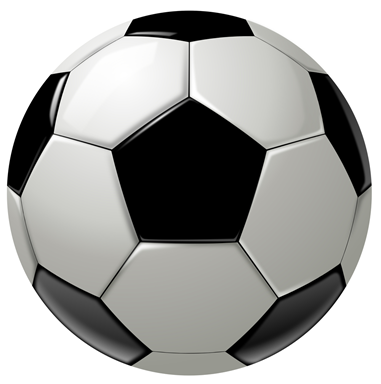}
            \caption{A Truncated Icosahedron.}
            \label{subfig:soccerball}
        \end{subfigure}
        \hfill
        \begin{subfigure}[c]{0.45\textwidth}
            \includegraphics[width=\textwidth, keepaspectratio]{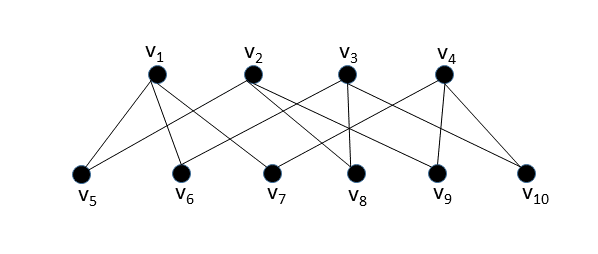}
            \caption{Graph with Identifying Code Set $\{v_1, v_2, v_3, v_4\}$.}
            \label{subfig:IC_1_Graph}
        \end{subfigure}
	\end{center}
	\caption{Satellites as sensors and soccer ball as a model of planet Earth.}
\end{figure}

In this work, we use a soccer ball as a model of the spheroid that is planet Earth.
A standard soccer ball is a {\em truncated icosahedron} with $12$ pentagonal and $20$ hexagonal patches~\cite{kotschick06} (shown in \cref{subfig:soccerball} as black and white patches).
We associate a patch on the surface of the ball with a region on the surface of the Earth. Accordingly, in our model the surface of the Earth is partitioned into $32$ regions. We assume that the coverage area of a satellite corresponds to a patch (region) and events are confined to a region. With such a framework, it is clear that with $32$ satellites (one per region), all the $32$ regions can be effectively monitored.

However, if we assume that the {\em impact} of an event in one region will {\em spill} into the neighboring regions, and as such there will be indicators of such events in neighboring regions, then a substantially lower number of satellites may be sufficient for effective monitoring of all the regions. 
As an example of impact of an event spilling over to neighboring regions, one can think of a situation where war breaking out in one region can trigger an exodus of refugees to the neighboring regions. 
As these sensors are expensive, one would like to deploy as few sensors as possible, subject to the constraint that all the regions can be effectively monitored. 
In this work, we propose to use {\em Identifying Codes}~\cite{karpov98} as a tool for determining how many sensors to deploy, and on which regions to employ them.

In particular, we show that ten satellites are {\em sufficient} to {\em effectively monitor} $32$ regions in the sense that, if an event breaks out in a region, that region can be {\em uniquely identified}. 
In fact there exists $26$ different ways of deploying ten satellites that will achieve the effective monitoring task. 
We also establish that the effective monitoring task cannot be accomplished by deployment of {\em fewer than} nine satellites.

Our approach to proving these results is two-fold.
First, we provide human-oriented, `traditional' pen-and-paper proofs to show that we need {\em at least} nine satellites for the task of monitoring the Earth, and that there exist {\em at least} $26$ different ways to deploy ten satellites for this task.
Then, we provide machine-oriented {\em formal} proofs to show that, in fact, we need at least {\em ten} satellites for the task of monitoring the Earth, and that the number of ways to deploy ten satellites for this purpose is {\em exactly} $26$.

Our choice for this approach is motivated by the ongoing discussion in the Mathematics community about the nature of proofs of mathematical conjectures.
In a recent commentary article~\cite{BBB+24} for the American Mathematical Society, Benzm\"uller proposed to present readers with 
\begin{itemize}
    \item a human-oriented, `traditional' proofs, aimed at providing insight into the `why' of the truth of the conjecture, and, ideally, inspiration for the discovery and proof of new mathematical truths, and
    \item a machine-oriented, formal proof, which should increase the trustworthiness of the result.
\end{itemize}
In this work, we indeed aim to do both.

We are enabled by the recent release of \veripbvtwo{}~\cite{GMM+20,GMNO22,BGMN22,EGMN20,GMN20,GN21}, a proof verifier that can formally verify the correctness {\em cutting planes}~\cite{Gom58} proofs of the unsatisfiability of sets of {\em pseudo-Boolean constraints}~\cite{RM21,BMH+21} (PB constraints, for short).
This allows us to capture the statement \textit{`Nine satellites are sufficient to monitor the Earth'} as a set of PB constraints, and {\em prove} that that statement is false.
We also use the PB solver and verifier to generate formal proofs that {\em ten} satellites {\em are} sufficient to monitor the Earth, and that there are exactly $26$ different ways of deploying those ten satellites.

The rest of this work is organised as follows.
We discuss existing work related to Identifying Codes, proof logging and verification in \cref{sec:related-work}.
Then, we discuss some notation and preliminaries regarding graphs, Identifying Codes and soccer balls, after which we describe the precise problem of satellite deployment in \cref{sec:problemformulation}.
We present a human-oriented traditional proof that ten satellites are sufficient to monitor the Earth in \cref{sec:upper-bound}, where we also show that are at least $26$ ways of deploying ten satellites to do so.
Similarly, \cref{sec:combinatorial-lower-bound} provides a human-oriented proof that we need {\em at least} nine satellites to do so.
We then move on to the formal proofs in this work.
\Cref{sec:machine-verifiable-proofs} provides some preliminaries on pseudo-Boolean solving, and describes the basic idea of proof logging for the cutting-planes proof system.
Then, in \cref{sec:pb-proof} we describe the steps we took to obtain a proof that nine satellites are not enough to monitor the Earth, but ten satellites are.
We also give a formal, machine-oriented proof that there are exactly different $26$ ways to deploy those ten satellites.
Finally, we conclude this work in \cref{sec:conclusion}.
\section{Related Work}
\label{sec:related-work}

We first briefly discuss existing work on the topic of Identifying Codes. 
Then, we provide a short history of machine-generated proofs in the context of combinatorial solving.

\subsection{Identifying Codes}

The study of Identifying Codes and its applications in sensor network domains has a history that spans two and a half decades.
Karpovsky \etal{}~\cite{karpov98} introduced the concept of Identifying Codes in 1998 and provided results for Identifying Codes for graphs with specific topologies, such as binary cubes and trees.
Charon \etal{} studied complexity issues related to computation of minimum Identifying Codes for graph and showed that in several types of graph, the problem is NP-hard \cite{charon2002,charon2003}. 
Auger showed that the problem can be solved in linear time if the graph happened to be a tree, but even for a planar graph the problem remains NP-complete~\cite{auger10}.

Using Identifying Codes, Laifenfeld \etal{} studied covering problems in \cite{laifen08} and joint monitoring and routing in wireless sensor networks in \cite{laifen09}. 
Basu \etal{} recently applied Identifying Codes in problems regarding the identification of criminals in social networks~\cite{BS21a} and the identification of spreaders of misinformation in online networks~\cite{BS21b}. 
Ray \etal{} in \cite{ray03} generalized the concept of Identifying Codes, to incorporate robustness properties to deal with faults in sensor networks. 

A classical approach to solving Identifying Code problems for networks, is to encode the problem as an {\em integer-linear program (ILP)} or {\em mixed-integer program (MIP)}~\cite{BS21a,BS21b} and to solve it with an off-the-shelf MIP solver.
Recently, Latour \etal{} proposed an alternative method, in which they reduce the problem of finding a particular type of Identifying Codes in a network to the computationally harder problem of finding an {\em Independent Support} of a Boolean formula, allowing for an exponentially smaller encoding of the problem~\cite{LSM23}.
Approximation algorithms for computation of Identifying Codes for some special types of graphs are presented in \cite{xiao98, suomela07}.

The problem of finding Identifying Codes in graphs is closely related to that of {\em graph-constrained group testing}~\cite{CKMS10}, which uses the monitoring of {\em paths} rather than {\em nodes} in the network to uniquely identify network failures.
This field, known as {\em network tomography}, has seen some recent advances, including, for example, quantitative studies of the maximum number of simultaneous failures that can be uniquely identified~\cite{MHS+17}.

Topological and combinatorial properties of soccer balls have been studied extensively in \cite{kotschick06}.

\subsection{Machine-Verifiable Formal Proofs}

Over the five decades, different communities (including, but not limited to, the Programming Languages community, the Mathematics community, and the Boolean Satisfiability Solving community) have discovered the merits of {\em formal, machine-readable} proofs of mathematical conjectures.
The generation and verification for such proofs is motivated by the observation that, in practice, mathematical proofs often turn out to be wrong (even if the conjecture they `prove' is actually correct), see, \eg{}, Lamport's website on the topic~\cite{Lam19}.

In a recent commentary article for the American Mathematical Society, Benzm\"uller expressed the dream that, in the future, math papers will contain proofs that {\em integrate} human-oriented `traditional' proofs (to provide intuition to the reader) and machine-oriented formal proofs (to provide trust in the correctness of the result).
While integration of these two is a noble goal, it is not the focus of this paper.
Hence, we provide both types of proof, sans integration.
Assuming that the reader has an intuitive understanding of what a `traditional' proof is, we now focus on providing a short history of the use of machine-verifiable proofs in the context of combinatorial solving.
\section{Problem Formulation}
\label{sec:problemformulation}

In this section, we first provide some preliminaries regarding graphs, and give the definitions of the variations of the Identifying Code Set problem that we study in this work.
We also reformulate the Identifying Code Set problem, presenting it as a variant of the Graph Coloring problem.
Then, we give a specification of the Soccer Ball Graph that we use as a model of planet Earth in this work.

\subsection{Graphs, Identifying Codes and Identifying Code Sets}

In this work, we study a combinatorial problem that we model on an {\em undirected graph} $G := (V,E)$, where $V$ denotes the set of {\em nodes} of the graph, and $E$ denotes the set of {\em edges}.
We use $N(v)$ to denote the {\em neighborhood} of a node $v \in V$, \ie{}, the set of nodes that are adjacent to $v$.
Similarly, we denote the {\em closed neighborhood} of a node $v \in V$ as $N^+(v) := \{v\} \cup N(v)$.

The notion of {\em Identifying Codes}~\cite{karpov98} has been established as a useful concept for optimizing sensor deployment in multiple domains. 
In this paper, we use Identifying Codes of the {\em simplest form}, meaning that we only consider the $1$-hop neighbourhood of a node to determine its identifying code. In the more general version, we might consider the $k$-hop neighbourhood, with $k \in \mathbb{N}^+$.
For the scope of this work, however, we define Identifying Codes as follows. 
\begin{definition}
    \label{def:ICS}
    A vertex set $V^\prime \subseteq V$ of a graph $G = (V, E)$ is called an {\em Identifying Code Set (ICS)} if, for all $v, w \in V$ with $v \neq w$, it holds that $N^+(v) \cap V^\prime \neq N^+(w) \cap V^\prime$.
\end{definition}
Using the notion of an Identifying Code Set, we define the define the following combinatorial decision and optimisation problems::
\begin{definition}
    \label{def:ICS-decision-problem}
    Given an undirected graph $G := (V,E)$ on nodes $V$ and $E$ and a set $V^\prime \subseteq V$, the {\em Identifying Code Set decision problem} asks to determine if $V^\prime$ an ICS of $G$.
\end{definition}
\begin{definition}
    \label{def:MICS}
    The {\em Minimum Identifying Code Set (MICS) problem} is to find an Identifying Code Set of {\em smallest cardinality}.
\end{definition}

The vertices of the set $V^\prime$ may be viewed as {\em alphabets} of the code, and the {\em string} made up with the alphabets of $N^+(v)$ may be viewed as the unique ``code'' for the node $v$. 
For instance, consider the graph $G = (V, E)$ shown in \cref{subfig:IC_1_Graph}. 
In this graph $V^\prime = \{v_1, v_2, v_3, v_4\}$ is an ICS, as it can be seen from \cref{tab:exampletable} that $N^+(v) \cap V^\prime$ is {\em unique} for all $v_i \in V$. 

\begin{table}[ht]
    \centering
    \caption{$N^+(v) \cap V^\prime$ results for all $v \in V$ for the graph in \cref{subfig:IC_1_Graph}.}
    \label{tab:exampletable}
    \begin{tabular} {ll}  \toprule
            $N^+(v_1) \cap V^\prime = \{v_1\}$ & $N^+(v_2) \cap V^\prime = \{v_2\}$ \\ 
            $N^+(v_3) \cap V^\prime = \{v_3\}$ & $N^+(v_4) \cap V^\prime = \{v_4\}$ \\ 
            $N^+(v_5) \cap V^\prime = \{v_1, v_2\}$ & $N^+(v_6) \cap V^\prime = \{v_1, v_3\}$ \\ 
            $N^+(v_7) \cap V^\prime = \{v_1, v_4\}$ & $N^+(v_8) \cap V^\prime = \{v_2, v_3\}$ \\ 
            $N^+(v_9) \cap V^\prime = \{v_2, v_4\}$ & $N^+(v_{10}) \cap V^\prime = \{v_3, v_4\}$ \\ \bottomrule
    \end{tabular}
\end{table}

\subsection{Graph Coloring Problems}
\label{subsec:GCS}

The MICS computation problem can be viewed as a novel variation of the classical Graph Coloring problem. 
We will refer to this version as the {\em Graph Coloring with Seepage (GCS) problem}. 
In the classical graph coloring problem, when a color is {\em assigned} (or injected) to a node, only that node is colored. 
The goal of the standard graph coloring problem to use as few distinct colors as possible such that (i) every node receives a color, and (ii) no two adjacent nodes of the graph have the same color.

In the GCS problem, when a color is assigned (or injected) to a node, not only that node receives the color, but also the color {\em seeps} into all the adjoining nodes. 
As a node $v_i$ may be adjacent to two other nodes $v_j$ and $v_k$ in the graph, if the color red is injected to $v_j$, not only $v_j$ will become red, but also $v_i$ will also become red as it is adjacent to $v_j$. 
Now if the color blue is injected to $v_k$, not only $v_k$ will become blue, but also the color blue will seep in to $v_i$ as it is adjacent to $v_k$. 
Since $v_i$ was already colored red (due to seepage from $v_j$), after color seepage from $v_k$, its color will be a {\em combination of red and blue}.  
At this point all three nodes  $v_i$, $v_j$ and $v_k$ have a color and all of them have distinct colors (red, blue and the combination of the two). 
{\em The color assigned to a node may be due to, (i) only injection to that node, (ii) only seepage from other adjoining nodes and (iii) a combination of injection and seepage}. 
We refer to the colors injected at the nodes as {\em atomic} colors.
The colors formed by the combination of two or more atomic colors are referred to as {\em composite} colors.
The colors injected at the nodes (atomic colors) are all {\em unique}. 
The goal of the GCS problem is to inject colors to as few nodes as possible, such that (i) every node receives a color, and (ii) no two nodes of the graph have the same color. 

Suppose that the node set $V'$ is an ICS of a graph $G = (V, E)$ and $|V'| = p$.  In this case if $p$ distinct colors are injected to $V'$ (one distinct atomic color to one node of $V'$ ), then as by the definition of ICS for all $v \in V$, if $N^+(v) \cap V'$ is unique, all nodes of  $G = (V, E)$ will have a unique color (either atomic or composite). Thus computation of MICS is equivalent to solving the GCS problem.

\subsection{The Soccer Ball Graph}

From the soccer ball, we construct a graph (referred to as a {\em Soccer Ball Graph (SBG)}) where each of the $32$ regions is represented as a node and two nodes have an edge between them if the corresponding regions share a boundary. 
The construction rules for the SBG are given in \cref{sec:problemformulation} and a two-dimensional layout of the SBG is shown in \cref{subfig:SBG}. 
Using a human-oriented, analytical proof, we establish that the upper and lower bounds of the MICS problem for the SBG are ten and nine respectively. 
Using a machine-oriented formal proof, we then establish that the lower bound of nine is not strict, and that there are in fact no solutions to the MICS problem for the SBG with a cardinality of less than ten.
We prove that there exist at least $26$ different Identifying Code Sets of size ten for the SBG, using a human-oriented, analytical proof, and establish that there are in fact {\em exactly} $26$ different Identifying Code Sets of size ten for the SBG, using a machine-oriented proof.

\begin{figure}[thb]
	\begin{center}
        \begin{subfigure}[c]{0.20\textwidth}
            \includegraphics[width=\textwidth, keepaspectratio]{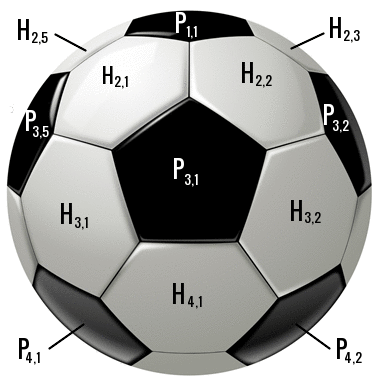}
            \caption{Labels of nodes corresponding to the patches on the soccer ball.}
            \label{subfig:footprint}
        \end{subfigure}
		\hfill
        \begin{subfigure}[c]{0.78\textwidth}
            \includegraphics[width=\textwidth, keepaspectratio]{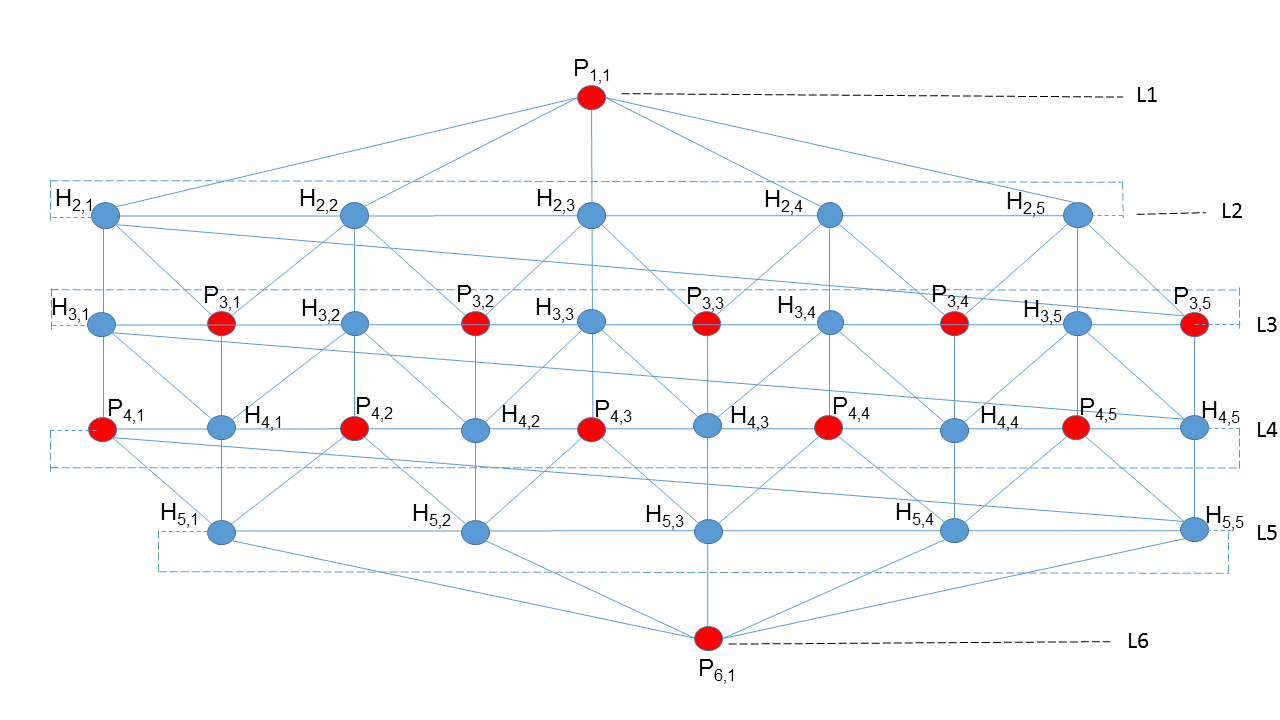}
            \caption{Soccer Ball Graph.}
            \label{subfig:SBG}
        \end{subfigure}
	\end{center}
	\caption{A soccer ball and the corresponding Soccer Ball Graph (SBG).}
	\label{fig:SBG}
\end{figure}

A Soccer Ball Graph (SBG) $G = (V, E)$ is defined in the following way. 

\begin{definition}
    \label{def:SBG}
    The graph comprises 32 nodes and 90 edges. The $32$ nodes correspond to $32$ patches ($20$ hexagonal and $12$ pentagonal) of the soccer ball and two nodes in the graph have an edge between them if the corresponding patches share a boundary. 
\end{definition}

A graph can have different layouts on a two dimensional plane. 
We show one layout of the SBG in \cref{subfig:SBG} where the nodes are labeled using a set of rules. 
The soccer ball, placed on a two dimensional plane (as shown in \cref{subfig:footprint}), has a pentagonal patch on top. 
There are six hexagonal patches adjacent to this pentagonal patch. 

We consider a {\em layering scheme}, where the node corresponding to the pentagonal patch on top is in Layer 1 (L1), the six nodes corresponding to six hexagonal patches adjacent to the pentagonal patch on top are in Layer 2 (L2) and so on. 
Following this layering scheme, all $32$ nodes can be assigned to six layers, L1 through L6, as shown in \cref{subfig:SBG}. 
In this scheme, one node is assigned to L1, five nodes to L2, ten nodes to L3, ten nodes to L4, five nodes to L5, and one node to L6. 
There is only one pentagonal node in layers L1 and L6 and we refer to these two nodes as $P_{1, 1}$ and $P_{6, 1}$ respectively. 
There are five hexagonal nodes in layers L2 and L5 and we refer to these nodes as $H_{2, j}$, with $1 \leq j \leq 5$ and $H_{5, j}$, with $1 \leq j \leq 5$, respectively.  
There are five hexagonal and five pentagonal nodes in layers L3 and L4 and we refer to these nodes as $H_{i, j}$ (with $i \in \{3, 4\}$ and $1 \leq j \leq 5$), and $P_{i, j}$ (with $i \in \{3, 4\}$ and $1 \leq j \leq 5$), respectively.

The vertex set $V$ of the SBG is divided into two subsets: a $P$ (for Pentagon) and $H$ (for Hexagon), with $12$ and $20$ members, respectively. 
It may be noted from \cref{fig:SBG} that $P$-type nodes appear only on layers $1$, $3$, $4$ and $6$ and $H$-type nodes appear only on layers $2$, $3$, $4$ and $5$. 
Hence, the SBG nodes are $V := P \cup H$, with $P = \{P_{1,1}\} \cup  \{P_{i,j} \mid 3 \leq i \leq 4, 1 \leq j \leq 5 \} \cup \{P_{6,1}\}$ and $H = \{ H_{i,j} \mid 2 \leq i \leq 5, 1 \leq j \leq 5 \}$.

The edge set $E$ of the SBG is divided into $17$ subsets: $E = \cup_{i=1}^{17} E_i$. 
Each subset is defined in \cref{tab:colorassignment}.

\begin{table}[htp]
	\centering
        \caption{Edge set definitions for the Soccer Ball Graph.}
	    \label{tab:colorassignment}
		\begin{tabular} {l}  \toprule
		\textbf{SBG Edge Construction} \\ \midrule
		$E_1 := \{(P_{1,1}, H_{2,j}) \mid 1 \leq j \leq 5\}$\\
		$E_2 := \{(P_{6,1}, H_{5,j}) \mid 1 \leq j \leq 5\}$\\
		$E_3 := \{(H_{i,j}, H_{i,(j+1)\bmod 5}) \mid i \in \{2,5\}, 1 \leq j \leq 5\}$\\
		$E_4 := \{(H_{3,j}, P_{3,j}) \mid 1 \leq j \leq 5\}$\\
		$E_5 := \{(P_{3,j}, H_{3,(j+1)\bmod 5}) \mid 1 \leq j \leq 5\}$\\
		$E_6 := \{(H_{4,j}, P_{4,(j+1)\bmod 5}) \mid 1 \leq j \leq 5\}$\\
		$E_7 := \{(P_{4,j}, H_{4,j}) \mid 1 \leq j \leq 5\}$\\
		$E_8 := \{(H_{2,j}, H_{3,j}) \mid 1 \leq j \leq 5\}$\\
		$E_9 := \{(H_{2,j}, P_{3,(j-1)\bmod 5}) \mid 1 \leq j \leq 5\}$\\
		$E_{10} := \{(H_{2,j}, P_{3,j}) \mid 1 \leq j \leq 5\}$\\
		$E_{11} := \{(H_{3,j}, P_{4,j}) \mid 1 \leq j \leq 5\}$\\
		$E_{12} := \{(H_{3,j}, H_{4,(j-1)\bmod 5}) \mid 1 \leq j \leq 5\}$\\
		$E_{13} := \{(H_{3,j}, H_{4,j}) \mid 1 \leq j \leq 5\}$\\
		$E_{14} := \{(P_{3,j}, H_{4,j}) \mid 1 \leq j \leq 5\}$\\
		$E_{15} := \{(H_{4,j}, H_{5,j}) \mid 1 \leq j \leq 5\}$ \\
		$E_{16} := \{(P_{4,j}, H_{5,j}) \mid 1 \leq j \leq 5\}$\\
		$E_{17} := \{(P_{4,j}, H_{5,(j-1)\bmod 5}) \mid 1 \leq j \leq 5\}$ \\ \bottomrule
		\end{tabular}
\end{table}

With the formal definition of the SBG complete, it may be observed that the problem of determining the fewest number of satellites necessary to uniquely identify the region (among $32$ regions) where a significant event has taken place is equivalent to computation of the Minimum Identifying Code Set problem for the SBG.

\section{Upper Bound of MICS of SBG}
\label{sec:upper-bound}

In this section, we first show that MICS of the SBG is at most ten and there exists at least $26$ ICSes of size ten.

\begin{theorem}
    \label{thrm:MICS-of-SBG-is-at-most-10}
	The MICS of SBG is at most ten.
\end{theorem}

\begin{proof}
Inject colors $A, B, C, D, E$ to the nodes $H_{2, j}, 1 \leq j \leq 5$ and colors $E, F, G, H, I, J$ to the nodes $H_{5, j}, 1 \leq j \leq 5$. Injection of ten different colors at these ten nodes, will cause color seepage to all other nodes of SBG. The color seepage will be constrained by the topological structure of the SBG. It may be verified that because of the constraint imposed by the SBG structure, and the fact that seepage takes place only to the neighbors of the node where a color is injected, the $32$ nodes of the SBG will have the color assignment shown in \cref{tab:color-assignment-after-seepage-in-the-SBG}. In the entries of \cref{tab:color-assignment-after-seepage-in-the-SBG}, $H_{2, 1}: A^*BE$ implies that the color $A$ was {\em injected} at the node $H_{2, 1}$ and the colors  $B$ and $E$ {\em seeped} into the node $H_{2, 1}$, from the adjacent nodes $H_{2, 2}$ and $H_{2, 5}$, where the colors $B$ and $E$ were injected. In general, if an alphabet $A$ through $E$ (representing distinct colors), appears {\em with} a * as a part of a string attached to a node (such as $H_{2, 1}$), it implies that the color was {\em injected} at that node. On the other hand, if an alphabet appears {\em without}  a * as a part of a string attached to a node, it implies that the color {\em seeped} into that node from one of the adjacent nodes. It may be verified that the color assignment to the nodes, as shown in \cref{tab:color-assignment-after-seepage-in-the-SBG} is {\em unique} (\ie{}, no two nodes have the same color or {\em strings} assigned to them).
\end{proof}

\begin{figure}[tb]
    \begin{center}
    \begin{subfigure}[c]{0.48\textwidth}
        \includegraphics[width=\textwidth, keepaspectratio]{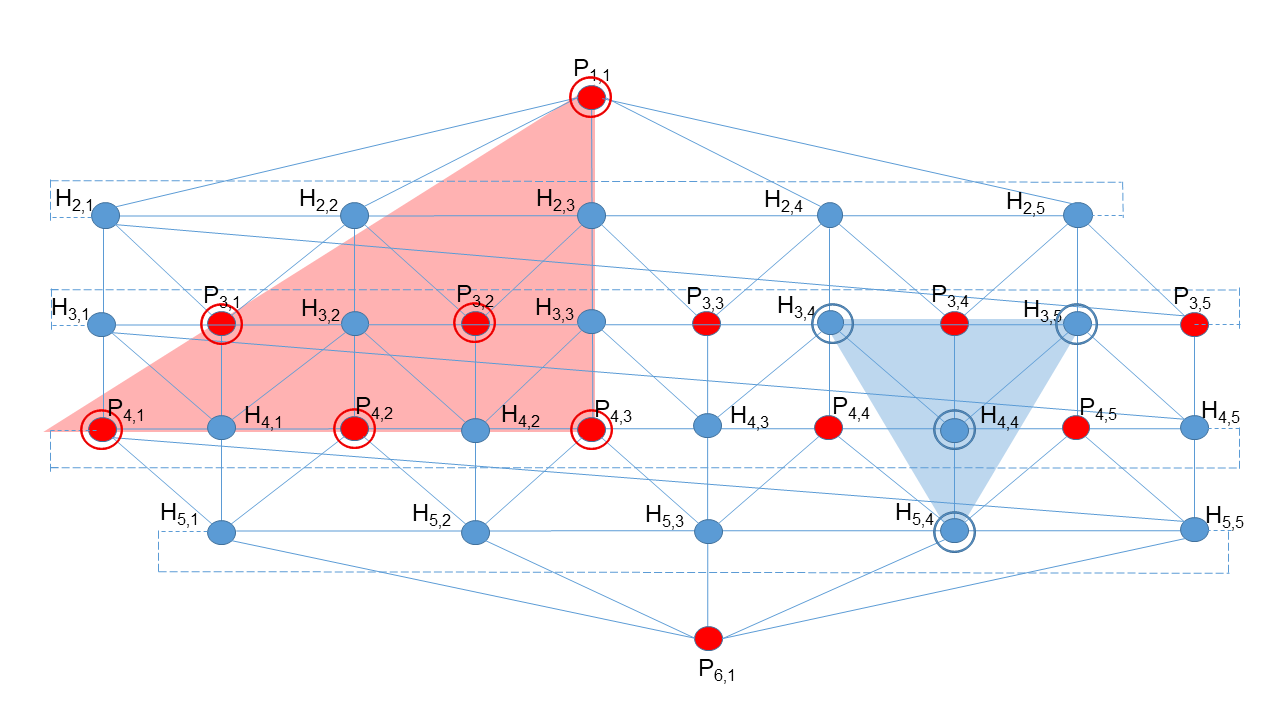}
        \caption{Class IIA Motif Assignment I.}
        \label{subfig:classIIA-assI}
    \end{subfigure}
    \hfill
    \begin{subfigure}[c]{0.48\textwidth}
        \includegraphics[width=\textwidth, keepaspectratio]{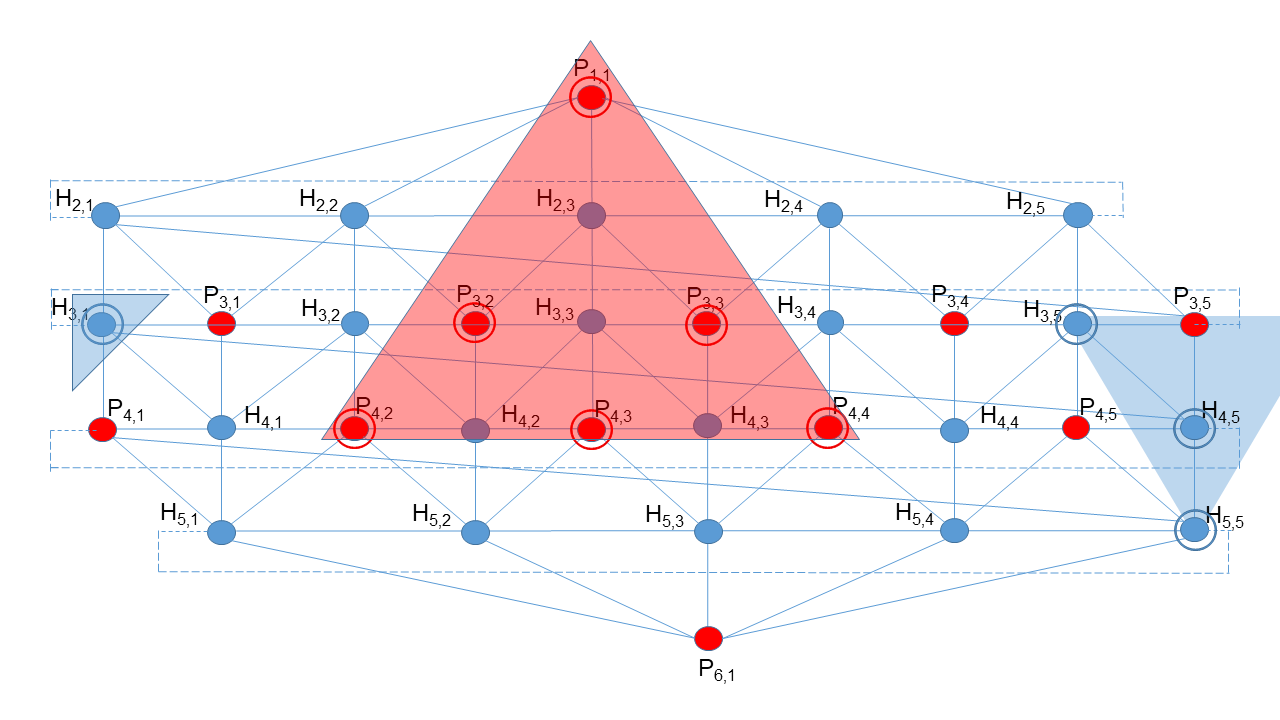}
        \caption{Class IIA Motif Assignment II (Assignment I shifted one position to the right).}
        \label{subfig:classIIA-assII}
    \end{subfigure}
    \end{center}
\caption{Examples of Color Assignments using Motif IIA. A circle drawn around a node indicates that that node is injected with color. We visualise the motifs by drawing red convex hulls around a $P$-type nodes that are injected with color, and blue convex hulls around all $H$-type nodes that are injected with color.}
\label{fig:classIIA}
\end{figure}

\begin{table}[ht]
	\centering
    \caption{Color assignment at nodes after seepage in the SBG, due to injection of Class I.}
	\label{tab:color-assignment-after-seepage-in-the-SBG}
		\begin{tabular} { l   l  l   l  }  \toprule
		Node: Color & Node: Color & Node: Color & Node: Color \\ \midrule
		$P_{1, 1}$: $ABCDE$ & $H_{2, 1}$: $A^*BE$ & $H_{2, 2}$: $AB^*C$ & $H_{2, 3}$: $BC^*D$ \\ 
		$H_{2, 4}$: $CD^*E$ & $H_{2, 5}$: $DE^*A$ & $H_{3, 1}$: $A$ & $P_{3, 1}$: $AB$ \\ 	
		$H_{3, 2}$: $B$ & $P_{3, 2}$: $BC$ & $H_{3, 3}$ : $C$ &  $P_{3, 3}$: $CD$ \\ 
		$H_{3, 4}$: $D$ & $P_{3, 4}$: $DE$ & $H_{3, 5}$ : $E$ & $P_{3, 5}$: $AE$ \\ 
		$P_{4, 1}$: $JF$ & $H_{4, 1}$: $F$ & $P_{4, 2}$ : $FG$ & $H_{4, 2}$: $G$ \\ 
		$P_{4, 3}$: $GH$ & $H_{4, 3}$: $H$ & $P_{4, 4}$ : $HI$ & $H_{4, 4}$: $I$ \\ 	
		$P_{4, 5}$: $IJ$ & $H_{4, 5}$: $J$ & $H_{5, 1}$: $JF^*G$ & $H_{5, 2}$: $FG^*H$ \\ 
		$H_{5, 3}$: $GH^*I$ & $H_{5, 4}$: $HI^*J$ & $H_{5, 5}$: $IJ^*F$ & $P_{6, 1}$: $FGHIJ$ \\ \bottomrule
		\end{tabular}
\end{table}

\begin{theorem}
\label{thrm:at-least-26}
There exist at least $26$ distinct size-10 ICSes of the SBG. 
\end{theorem}

\begin{proof}
The $26$ different ways in which ten colors can be injected into ten nodes of the SBG such that every node of the SBG receives a unique color can be divided into four classes.
\begin{description}
    \item[Class I] Inject colors $A, B, C, D, E$ to the nodes $H_{2, j}, 1 \leq j \leq 5$ and colors $E, F, G, H, I, J$ to the nodes $H_{5, j}, 1 \leq j \leq 5$. As shown in \cref{tab:color-assignment-after-seepage-in-the-SBG}, such an injection ensures that each of the $32$ nodes of the SBG receives a unique color. It may be observed that the node set where the colors are injected in this Class all have degree six, corresponding to hexagonal patches on the surface of the soccer ball. Only one ICS of the $26$ belongs to Class I.

\begin{table}[htp]
    \centering
    \caption{Color assignment at nodes after seepage for Class II ICS.}
    \label{tab:color-assignment-after-seepage-for-Class-II-ICS}
    \begin{tabular} { l l }  \toprule 
        \textbf{Node} & \textbf{Color} \\ \midrule
        $P_{1, 1}$ & $P_{1, 1}^c$ \\
        $H_{2, j}$ & $P_{1, 1}^cP_{3,j}^c$ \\ 
        $H_{2, (j+1) \bmod 5}$ & $P_{1, 1}^cP_{3, j}^cP_{3, (j + 1) \bmod 5}^c$ \\
        $H_{2, (j+2) \bmod 5}$& $P_{1, 1}^cP_{3,(j+1) \bmod 5}^c$ \\ 
        $H_{2,(j+3) \bmod 5}$& $P_{1, 1}^cH_{3,(j+3) \bmod 5}^c$ \\ 
        $H_{2,(j+4) \bmod 5}$& $P_{1, 1}^cH_{3, (j + 4) \bmod 5}^c$ \\ 
        $H_{3, j}$& $P_{3,j}^cP_{4,j}^c$ \\ 
        $P_{3, j}$ & $P_{3,j}^c$ \\ 	
        $H_{3, (j+1) \bmod 5}$ &$P_{3,j}^cP_{3, (j+1) \bmod 5}^cP_{4, (j+1) \bmod 5}^c$ \\ 
        $P_{3, (j+1) \bmod 5}$& $P_{3, (j+1) \bmod 5}^c$ \\ 	
        $H_{3, (j+2) \bmod 5}$& $P_{3, (j+1) \bmod 5}^cP_{4, (j+2) \bmod 5}^c$ \\ 
        $P_{3, (j+2) \bmod 5}$& $H_{3, (j+3) \bmod 5}^c$ \\ 
        $H_{3, (j+3) \bmod 5}$& $H_{3, (j+3) \bmod 5}^cH_{4, (j+3) \bmod 5}^c$ \\ 
        $P_{3, (j+3) \bmod 5}$& $H_{3, (j+3) \bmod 5}^cH_{4, (j+3) \bmod 5}^cH_{3, (j+4) \bmod 5}^c$ \\ 
        $H_{3, (j+4) \bmod 5}$& $H_{3, (j+4) \bmod 5}^cH_{4, (j+3) \bmod 5}^c$ \\ 
        $P_{3, (j+4) \bmod 5}$& $H_{4, (j+4) \bmod 5}^c$ \\ 
        $P_{4, j}$& $P_{4, j}^c$ \\ 
        $H_{4, j}$ & $P_{3,j}^cP_{4,j}^cP_{4,(j+1) \bmod 5}^c$ \\ 
        $P_{4, (j+1) \bmod 5}$& $P_{4, (j+1) \bmod 5}^c$ \\ 
        $H_{4, (j+1) \bmod 5}$& $P_{3, (j+1) \bmod 5}^cP_{4, (j+1) \bmod 5}^cP_{4, (j+2) \bmod 5}^c$ \\ 
        $P_{4, (j+2) \bmod 5}$& $P_{4, (j+2) \bmod 5}^c$ \\ 
        $H_{4, (j+2) \bmod 5}$& $P_{4, (j+2) \bmod 5}^cH_{3, (j+3) \bmod 5}^c$ \\  
        $P_{4, (j+3) \bmod 5}$& $H_{3, (j+3) \bmod 5}^cH_{4, (j+3) \bmod 5}^cH_{5, (j+3) \bmod 5}^c$ \\ 
        $H_{4, (j+3) \bmod 5}$& $H_{4, (j+3) \bmod 5}^c H_{3, (j+3) \bmod 5}^c H_{3, (j+4) \bmod 5}^c H_{5, (j+3) \bmod 5}^c$ \\ 
        $P_{4, (j+4) \bmod 5}$& $H_{4, (j+4) \bmod 5}^cH_{3, (j+3) \bmod 5}^cH_{5, (j+3) \bmod 5}^c$ \\ 
        $H_{4, (j+4) \bmod 5}$& $P_{4, j}^cH_{4, (j+4) \bmod 5}^c$ \\ 
        $H_{5, j}$& $P_{4, j}^cP_{4, (j+1) \bmod 5}^c$ \\ 
        $H_{5, (j+1) \bmod 5}$& $P_{4, (j+1) \bmod 5}^cP_{4, (j+2) \bmod 5}^c$ \\ 
        $H_{5, (j+2) \bmod 5}$& $P_{4, (j+2) \bmod 5}^cH_{5, (j+3) \bmod 5}^c$ \\ 
        $H_{5, (j+3) \bmod 5}$& $H_{5, (j+3) \bmod 5}^cH_{4, (j+3) \bmod 5}^c$ \\  
        $H_{5, (j+4) \bmod 5}$& $P_{4, j}^cH_{5, (j+3) \bmod 5}^c$ \\ 
        $P_{6, 1}$& $H_{5, (j+3) \bmod 5}^c$ \\ \bottomrule
    \end{tabular}
\end{table}

    \item[Class II] The node set where the colors injected are in this Class is made up of six nodes of degree five (corresponding to pentagonal patches of the soccer ball) and four nodes of degree six (corresponding to hexagonal patches of the soccer ball). 
    This Class can be subdivided into two sub classes and we refer to them as {\em Class II-A} and {\em Class II-B}, respectively. 
    As shown in \cref{subfig:SBG}, the SBG graph is somewhat symmetric in the sense that the layers $4$, $5$ and $6$ are close to being mirror images of layers $1$, $2$ and $3$.
    Because of this symmetry, the Class II-A color injections are mirror images of the Class II-B color injection. Accordingly, we focus our discussion primarily on Class II-A, as color injection for class II-B be can be obtained easily from color injection in Class II-A. We introduce the notion of a {\em motif}, and by motif we imply a set of either P-type (degree five) or H-type (degree six) nodes. It will be clear from further discussion that the Class II-A solutions comprise of one P-type motif and one H-type motif. 
    These two motifs {\em complement} each other to produce a solution together. 
    The motif pairs can be {\em translated} along the structure of the SBG to produce a set of five solutions that make up the Class II-A. 
    The five solutions that make up the Class II-B can be constructed in a similar fashion.
    For the ICS that belong to Class II, the P-type motif is made up of the set of six nodes $\{P_{1,1}, P_{3,j}, P_{3,(j+1)\bmod 5}, P_{4,j}, P_{4,(j+1)\bmod 5}, P_{4,(j+2)\bmod 5}\}$. 
    The H-type motif that complements the P-type motif is made up of the set of four nodes $\{H_{3,(j+3)\bmod 5}, H_{3,(j+4)\bmod 5}, H_{4,(j+3)\bmod 5}, H_{5,(j+3)\bmod 5}\}$. 
    One complete solution (\ie{}, ICS) is obtained by choosing a value of $j, 1 \leq j \leq 5$.
    \Cref{subfig:classIIA-assI,subfig:classIIA-assII} show the solutions with $j = 1$ and $j = 2$, respectively.  
    As shown in \cref{fig:classIIA}, changing the index $j$ from $1$ to $2$, has the effect of translating the motif along the structure of the SBG. 
    By changing $j$ from $1$ through $5$ (\ie{}, translating the motif $5$ times), $5$ different ICSes can be constructed.
    The colors that are associated with the nodes of the SBG, if they are injected at the motif nodes, are shown in \cref{tab:color-assignment-after-seepage-for-Class-II-ICS}. 
    The first column of the table indicates the node and the second column provides the color assigned to that node. 
    For example, in row 3 of \cref{tab:color-assignment-after-seepage-for-Class-II-ICS}, the node $H_{2, (j+1) \bmod 5}$ receives the colors injected at motif nodes $P_{1,1}, P_{3, j}, P_{3, (j+1) \bmod 5}$ and is denoted by $P_{1,1}^cP_{3, j}^cP_{3, (j+1) \bmod 5}^c$. 
    It may be verified that every node of the SBG has a {\em color associated with it and no two nodes have the same color assignment}.

\begin{figure}[htb]
    \begin{center}
    \begin{subfigure}[c]{0.48\textwidth}
        \includegraphics[width=\textwidth, keepaspectratio]{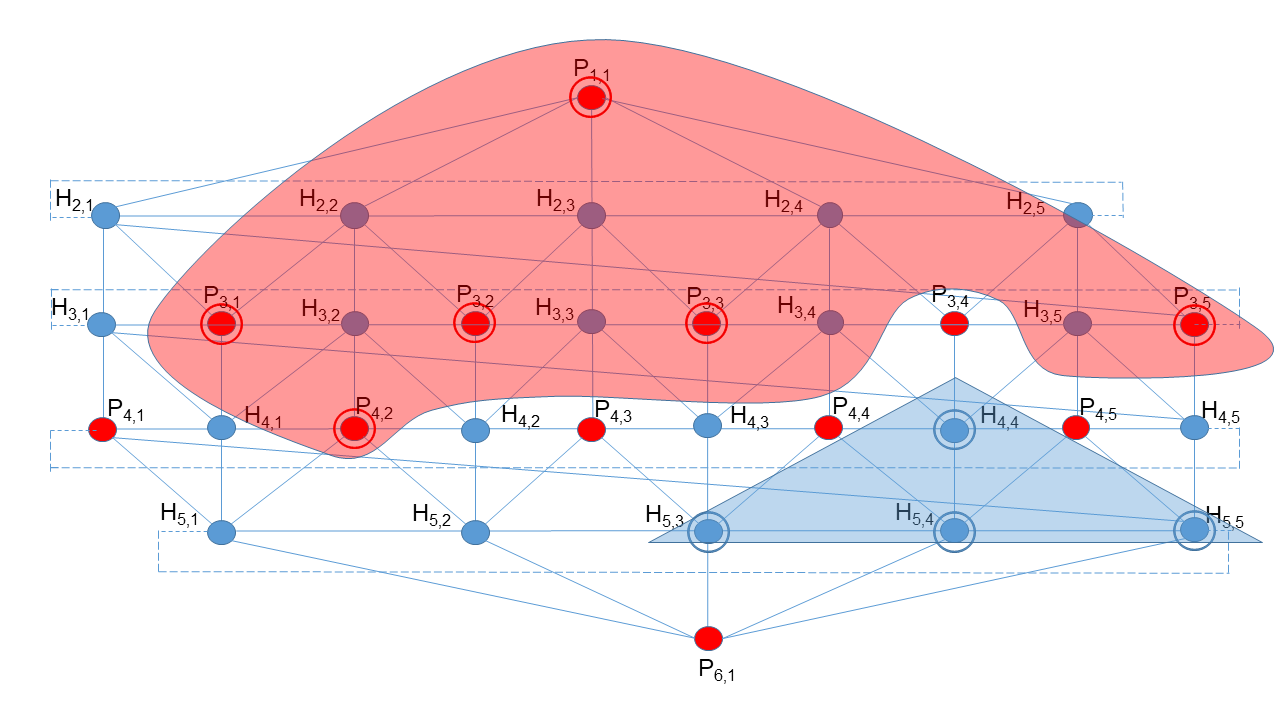}
        \caption{Class III-A Motif Assignment.}
        \label{subfig:class-IIIA}
    \end{subfigure}
    \hfill
    \begin{subfigure}[c]{0.48\textwidth}
        \includegraphics[width=\textwidth, keepaspectratio]{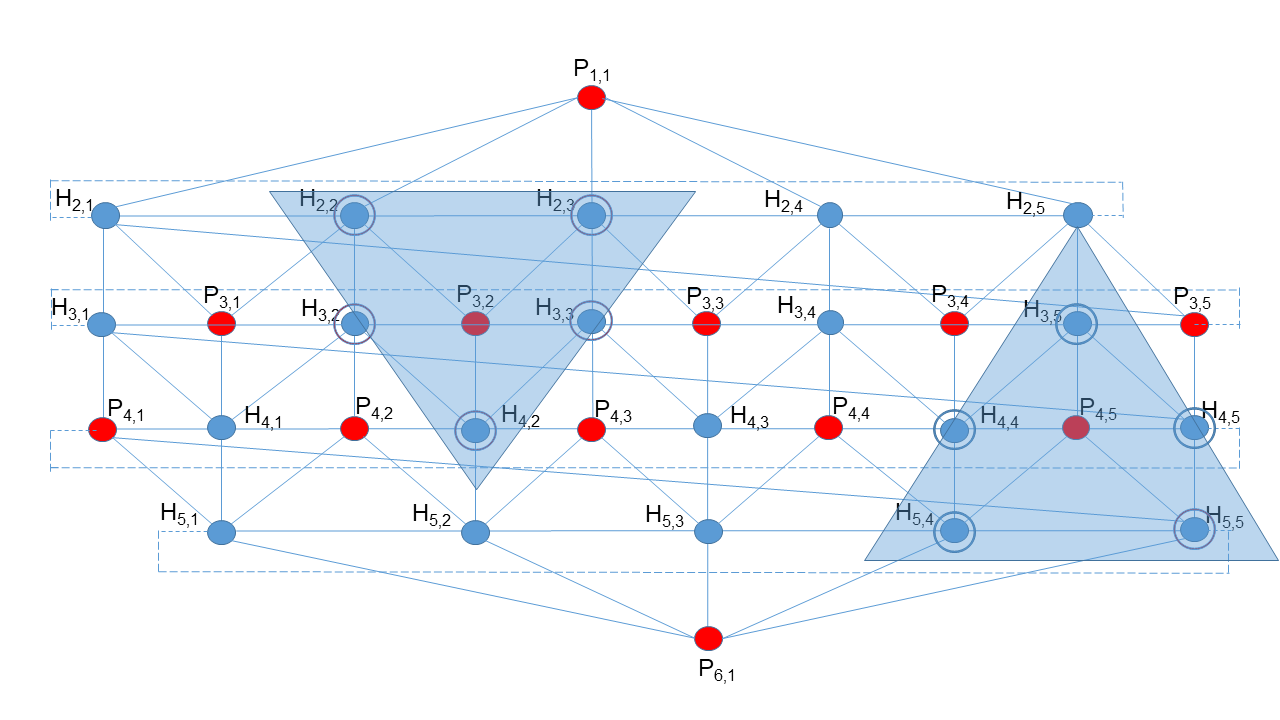}
        \caption{Class IV Motif Assignment.}
        \label{subfig:class-IV}
    \end{subfigure}
    \end{center}
    \caption{Examples of Color Assignments using Motifs IIIA and IV.}
    \label{fig:color-assignment-examples}
\end{figure}

\begin{table}[htb]
    \centering
    \caption{Color assignment at nodes after seepage for Class III ICS.}
    \label{tab:color-assignment-after-seepage-for-Class-III-ICS}
    \begin{tabular} { ll }  \toprule
        \textbf{Node} & \textbf{Color} \\ \midrule
        $P_{1, 1}$ & $P_{1,1}^c$ \\ 
        $H_{2, j}$ & $P_{1,1}^cP_{3,j}^cP_{3, (j+4) \bmod 5}^c$ \\  
        $H_{2, (j+1) \bmod 5}$ & $P_{1,1}^cP_{3,j}^cP_{3,(j+1) \bmod 5}^c$ \\ 
        $H_{2, (j+2) \bmod 5}$& $P_{1,1}^cP_{3,(j+1) \bmod 5}^cP_{3,(j+2) \bmod 5}^c$ \\ 
        $H_{2,(j+3) \bmod 5}$& $P_{1, 1}^cP_{3,(j+2) \bmod 5}^c$ \\ 
        $H_{2,(j+4) \bmod 5}$& $P_{1, 1}^cP_{3,(j+4) \bmod 5}^c$ \\ 	
        $H_{3, j}$& $P_{3, j}^cP_{3, (j+4) \bmod 5}^c$ \\ 
        $P_{3, j}$ & $P_{3, j}^c$ \\ 		
        $H_{3, (j+1) \bmod 5}$& $P_{3,j}^cP_{3,(j+1) \bmod 5}^cP_{4,(j+1) \bmod 5}^c$ \\ 
        $P_{3, (j+1) \bmod 5}$& $P_{3, (j+1) \bmod 5}^c$ \\ 		
        $H_{3, (j+2) \bmod 5}$& $P_{3, (j+1) \bmod 5}^cP_{3, (j+2) \bmod 5}^c$ \\ 
        $P_{3, (j+2) \bmod 5}$& $P_{3, (j+2) \bmod 5}^c$ \\ 
        $H_{3, (j+3) \bmod 5}$& $P_{3, (j+2) \bmod 5}^cH_{4, (j+3) \bmod 5}^c$ \\ 
        $P_{3, (j+3) \bmod 5}$& $H_{4, (j+3) \bmod 5}^c$ \\ 
        $H_{3, (j+4) \bmod 5}$& $H_{4, (j+3) \bmod 5}^cP_{3, (j+4) \bmod 5}^c$ \\ 
        $P_{3, (j+4) \bmod 5}$& $P_{3, (j+4) \bmod 5}^c$ \\ 
        $P_{4, j}$& $H_{5, (j+4) \bmod 5}^c$ \\ 
        $H_{4, j}$ & $P_{4, (j+1) \bmod 5}^cP_{3, j}^c$ \\ 
        $P_{4, (j+1) \bmod 5}$& $P_{4, (j+1) \bmod 5}^c$ \\ 
        $H_{4, (j+1) \bmod 5}$& $P_{4, (j+1) \bmod 5}^cP_{3, (j+1) \bmod 5}^c$ \\ 
        $P_{4, (j+2) \bmod 5}$& $H_{5, (j+2) \bmod 5}^c$ \\ 
        $H_{4, (j+2) \bmod 5}$& $P_{3, (j+2) \bmod 5}^cH_{5, (j+2) \bmod 5}^c$ \\  
        $P_{4, (j+3) \bmod 5}$& $H_{4, (j+3) \bmod 5}^cH_{5, (j+2) \bmod 5}^cH_{5, (j+3) \bmod 5}^c$ \\ 
        $H_{4, (j+3) \bmod 5}$& $H_{4, (j+3) \bmod 5}^cH_{5, (j+3) \bmod 5}^c$ \\ 	
        $P_{4, (j+4) \bmod 5}$& $H_{4, (j+3) \bmod 5}^cH_{5, (j+3) \bmod 5}^cH_{5, (j+4) \bmod 5}^c$ \\ 
        $H_{4, (j+4) \bmod 5}$& $P_{3, (j+4) \bmod 5}^cH_{5, (j+4) \bmod 5}^c$ \\ 
        $H_{5, j}$& $H_{5, (j+4) \bmod 5}^cP_{4, (j+1) \bmod 5}^c$ \\ 
        $H_{5, (j+1) \bmod 5}$& $H_{5, (j+2) \bmod 5}^cP_{4, (j+1) \bmod 5}^c$ \\ 
        $H_{5, (j+2) \bmod 5}$& $H_{5, (j+2) \bmod 5}^cH_{5, (j+3) \bmod 5}^c$ \\ 
        $H_{5, (j+3) \bmod 5}$& $H_{5, (j+3) \bmod 5}^c H_{5, (j+4) \bmod 5}^c H_{5, (j+2) \bmod 5}^c H_{4, (j+3) \bmod 5}^c$ \\  
        $H_{5, (j+4) \bmod 5}$& $H_{5, (j+4) \bmod 5}^cH_{5, (j+3) \bmod 5}^c$ \\   
        $P_{6, 1}$& $H_{5, (j+2) \bmod 5}^cH_{5, (j+3) \bmod 5}^cH_{5, (j+4) \bmod 5}^c$ \\ \bottomrule
    \end{tabular}
\end{table}

    \item[Class III] As in Class II, the Class III ICS is made up of six nodes of degree five and four nodes of degree six. 
    Like Class II, this Class also can be subdivided into two sub-classes and we will refer to them as Class III-A and Class III-B, respectively. 
    We restrict our discussion on Class III-A, as color injection for Class III-B can be obtained as a mirror image of Class III-A.
    It will be clear from further discussion that, as in Class II, the Class III solutions also comprise one $P$-type motif and one $H$-type motif and they complement each other to produce a solution together. 
    As in Class II, the motif pairs can be translated along the structure of the SBG to produce a set of five solutions that make up the Class III-A. 
    The five solutions that make up the Class III-B can be constructed in a similar fashion.
    
    \item[Class IV] \sloppy The $P$-type motif is made up of six nodes: $\{P_{1,1}, P_{3,j}, P_{3,(j+1)\bmod 5}, P_{3,(j+2) \bmod 5}, P_{3,(j+4)\bmod 5},\allowbreak  P_{4,(j+1)\bmod 5}\}$. 
    As shown in \cref{subfig:classIIA-assI}, changing the index $j$ from 1 to 5, has the effect of translating the motif along the structure of the SBG. 
    The $H$-type motif that complements the P-type motif is made up of the set of four nodes $\{H_{4,(j+3)\bmod 5}, H_{5,(j+2)\bmod 5}, H_{5,(j+3)\bmod 5}, H_{5,(j+4)\bmod 5}\}$. 
    One complete solution is obtained by choosing a value of $j, 1 \leq j \leq 5$. 
    By moving the $P$-type and $H$-type motifs in tandem by changing the value of the index from $1$ to $5$, five different solutions can be obtained.
    The colors that will be associated with the nodes of the SBG, if the they are injected at the motif nodes, are shown in \cref{tab:color-assignment-after-seepage-for-Class-III-ICS}. 
    It may be verified that every node of the SBG has a {\em color associated with it and no two nodes have the same color assignment}.
    
    \item[Class IV] As in Class I, the Class IV ICS are made up of 10 nodes of degree six (\ie{}, the nodes corresponding to hexagonal patches). This Class comprises of five ICS and cannot be subdivided like in Classes II and III.
    This class comprises two H-type motifs made up of five hexagonal nodes each. 
    The first motif comprises $\{H_{2, (j+1) \bmod 5}, H_{2, (j+2) \bmod 5}, H_{3, (j+1) \bmod 5}, H_{3, (j+2) \bmod 5}, H_{4, (j+1) \bmod 5}\}$. 
    The other motif comprises $\{H_{3, (j+4) \bmod 5}, H_{4, (j+3) \bmod 5}, H_{4, (j+4) \bmod 5}, H_{5, (j+3) \bmod 5}, H_{5, (j+4) \bmod 5}\}$. 
    As shown in \cref{subfig:classIIA-assII}, changing the index $j$ from $1$ to $5$, has the effect of translating the motif along the structure of the SBG. 
    One complete solution is obtained by choosing a value of $j, 1 \leq j \leq 5$. 
    By moving two H-type motifs in tandem, changing the value of the index from $1$ to $5$, five different solutions can be obtained. 
    As in Classes II and III, if colors are injected at the motif nodes, then every node of the SBG will have a {\em color associated with it and no two nodes will have the same color assignment}. 
\end{description}
This concludes proof of \cref{thrm:at-least-26}.
\end{proof}

\section{Combinatorial Lower Bound of MICS of SBG}
\label{sec:combinatorial-lower-bound}

In \cref{fig:SBG}, we have provided a {\em layered} representation of the SBG, where 32 nodes of the SBG is placed in six layers, indicated by L1 through L6. The layers L1 through L3 constitute the {\em top half} of the SBG and the layers L4 through L6 constitute the {\em bottom half}. As the two halves are symmetric, similar argument can be applied to both of them.

\begin{lemma}
\label{lem:4bound}
A MICS must select at least $4$ nodes from each half. In other words, at least $4$ distinct colors need to be injected in each half.
\end{lemma}

\begin{proof}
We provide arguments for the top half of the SBG, and we first show that three distinct colors are necessary to ensure that each node in top half receives a distinct color (either through injection, seepage or combination of the two). Consider layers  L1 and L2. No matter which nodes in bottom half are injected with colors, these colors will not seep into the nodes in L1 and L2 (colors seep only to adjacent nodes), coloring in bottom half nodes will not affect the colors associated with nodes in L1 or L2. Since L1 and L2 have six nodes, six distinct colors need to be associated with them.
It can be easily verified that in the SBG, in order to ensure distinct colors to each one of the six nodes in L1 and L2 at least three colors must be injected to three nodes in the top half of the SBG.  Clearly at least three nodes must be selected from each half so that nodes in L1, L2, L5 and L6 receive distinct colors. With injection of three colors, up to $2^3 - 1 = 7$ colors (excluding an empty combination) can be generated.

Next we show that three colors are not sufficient to color L1 and L2. Without loss of generality, we use alphabets $\{A,B,C\}$ to represent three colors. As mentioned above, $7$ distinct colors can be generated with these three colors (three primary and four composite), $\{A, B, C, AB, AC, BC, ABC\}$. Simple counting shows that each alphabet (color) appears exactly $4$ times. Suppose there is a proper injection using $A, B, C$ that ensures all nodes in L1 and L2 received distinct colors. Since seven distinct colors can be generated with three primary colors, and L1 and L2 has only six nodes, it implies that one of the seven colors (primary or composite) is not used while coloring the nodes of L1 and L2. This implies that at least one of the alphabets $A, B, C$ is appearing three times instead of four in the alphabet strings (representing colors) associated with the nodes of L1 and L2. Without loss of generality, {\em we assume that color $A$ is appearing 3 times}. There are four possible locations for injection of color $A$ in the top half of the SBG. In the following, we examine them all.

\begin{enumerate}
  \item \emph{$A$ is injected on L1, \ie{}, at $P_{1,1}$. $A$ would then appear at all nodes in L1 and L2, making its appearance six times, contradicting the assumption. 
  \item $A$ is injected on L2, \ie{}, one of $H_{2,i} (1 \leq i \leq 5$) nodes. Thus, $A$ appears four times (three nodes in L2 and one node in L1) contradicting the assumption. 
  \item $A$ is injected on one of the hexagonal nodes L3, \ie{}, one of the $H_{3,i}, 1 \leq i \leq 5$. Since $H_{3,i}$  has only one neighbor in on L1 and L2 ($H_{2,i}$), in this case $A$ will appear only on one node in L2, making its appearance one time, contradicting the assumption. 
  \item $A$ is injected on one of the pentagonal nodes L3, \ie{}, one of the $P_{3,i}, 1 \leq i \leq 5$. Since $P_{3,i}$ has only two neighbors in on L1 and L2 ($H_{2,(i-1)\mod 5}$ and $H_{2,i}$), in this case $A$ will appear only on two nodes in L2, making its appearance two times, contradicting the assumption.}
\end{enumerate}

As there is no location for injection of $A$, we can conclude that $3$ colors are inadequate to ensure that all nodes in L1 and L2 receive a unique color. Similar arguments can be made for coloring of nodes in L5 and L6. Therefore the lower bound of MICS for the SBG must be at least $4 + 4 = 8$.
\end{proof}

\begin{lemma}
\label{lem:at-least-9}
MICS of the SBG is at least 9.
\end{lemma}

\begin{proof}
In GCS problem (which is equivalent to the MICS problem), each node is assigned a color, which may be a primary or a composite color. A primary color is indicated by one alphabet and a composite color by a {\em string} of alphabets. The number of alphabets that appear in a string determines the {\em length} of that string. We establish the lemma by providing arguments based on the sum of the length of strings associated with each one of the 32 nodes of the SBG. We will refer to the sum of the length of strings associated with each one of the 32 nodes of the SBG as ``total string length''.

We use the term ``valid injection'' to imply an injection of colors to the nodes that ensures that all 32 nodes of the SBG receive a distinct color. Suppose there exists more than one valid injections using eight colors. Among the set of all valid injections, we consider the one whose total string length is minimum. The lower bound of the total string length for a valid injection with eight colors is 56. This is true, as with eight injected colors, at most eight nodes of the SBG can have associated strings of length one, and the remaining 24 nodes must have strings of length at least two. Thus the lower bound on the string length must be $8 \times 1 + 24 \times 2 = 56$.
The upper bound on the total string length with injection of at most eight colors is also 56. This is true for the following reason.
If a color is injected on a hexagonal node, then it will appear seven times (six neighbors and the node itself). Similarly for a pentagonal node, the color will appear six times. Therefore, the upper bound of total string length is $7 \times 8 = 56$.
It may be noted that the total string length is $56$ if and only if all colors are injected on hexagonal nodes. 
However, it is impossible to achieve a valid injection by injecting eight colors only on hexagonal nodes. Consider the top half of the SBG. In order to color nodes on L1, at least one color, say $A$, must be injected on one node on L2. Without loss of generality, we assume that $A$ is injected on $H_{2,i}, 1 \leq i \leq 5$. We consider two scenarios:

\begin{enumerate}
  \item No other color is injected at the nodes on L2. In this case, the other colors are injected at three hexagonal nodes on L3. Because of injection of $A$ at $H_{2, i}$, after seepage, all six adjacent nodes, $P_{1, 1}$, $H_{2, (i - 1)\mod 5}$, $H_{2, (i + 1) \mod{5}}$, $H_{3, (i - 1)\mod 5}$, $P_{3, i}$, $H_{3, i}$, will have color $A$.  In order ensure that all these nodes have distinct colors, colors must be injected on $H_{3, (i - 1)\mod 5}$, $H_{3, i}$, $H_{3, (i + 1)\mod 5}$.  However, if such an injection is made, the nodes $H_{2, (i + 2)\mod 5}$ and $H_{2, (i + 3)\mod 5}$ will not receive any color. Accordingly, such an injection will not be a valid injection. 
	 \item One or more colors are injected at the nodes on L2. Suppose a different color $B$ is injected at a node different from $H_{2,i}$ (where color $A$ is injected). Due to the SBG topology, no matter which node on L2 is injected with $B$,  one node on L2 and the node on L1 must have color $AB$ after seepage. In order to ensure distinct colors on these two nodes, a third color $C$ must be injected on another node. After injection of $C$, one of the two nodes that had the color $AB$ before injection of $C$, will have the color $AB$ and the other will have $ABC$. However, if one node has a string of length three, the lower bound of the total string length can longer be 56. It has to be at least 57, thus exceeding the upper bound (56), that is possible with injection of at most eight colors.
\end{enumerate}
\end{proof}

\begin{theorem}
The lower bound of MICS of the SBG is at least nine. In other words, eight colors are insufficient to ensure that all nodes of the the SBG receives a distinct color.
\end{theorem}

\begin{proof}
Follows from \cref{lem:4bound,lem:at-least-9}.
\end{proof}
\FloatBarrier
\section{Machine-Verifiable Proofs of Unsatisfiability}
\label{sec:machine-verifiable-proofs}

We can prove that the strict lower bound of MICS of SBG is ten, by using a {\em machine-verifiable proof of unsatisfiability} of sets of {\em pseudo-Boolean constraints} (`PB constraints', for short).
In this section, we first review some basic concepts and notation of pseudo-Boolean satisfiability. 
Then, we give a brief introduction to, and motivation for the use of, machine-verifiable proofs of unsatisfiability.
We then describe the basic principles of {\em cutting-planes}, a proof system for creating proofs of unsatisfiability of PB formulae, and briefly describe a way to encode those proofs, such that they can be read by an automated verifier.
In \cref{sec:pb-proof}, we describe the steps and the tools that we used to obtain a proof that the cardinality of the MICS of SBG is at least ten.

\subsection{Pseudo-Boolean Satisfiability}

For the formal proof part of this work, we will use {\em pseudo-Boolean formulae (PB formulae)} to model a 
We briefly recall some concepts and notation of pseudo-Boolean Satisfiability.
For a more detailed overview, we refer the reader to, \eg{}, Chapter 28 of the {\em Handbook of Satisfiability}~\cite{RM21,BMH+21}.
 
We denote a set of Boolean variables with the capital letter $X$ and denote individual Boolean variables with lowercase letters $x \in X$.
We denote truth values with $\true$ ({\em true}) and $\false$ ({\em false}).
A {\em literal} $\ell$ is a variable or its negation (\eg{}, $x$ or $\overline{x}$, respectively), where $\overline{x} = 1 - x$.
For simplicity, we write $\overline{\overline{x}} = x$.
We also note that $\overline{\true} = \false$ and $\overline{\false} = \true$.

In this work, we only consider {\em linear} pseudo-Boolean (PB) constraints.
A linear PB constraint on Boolean variables $X$ has, or can trivially be rewritten to have, the following form:
\begin{equation}
	C := \sum_i a_i \ell_i \bowtie b,
\end{equation}
where $a_i, b \in \mathbb{Z}$, and $\bowtie \: \in \left\{=, \leq, \geq, <, >\right\}$.
The $a_i$s are called the {\em coefficients} of the constraint.
The right-hand side $b$ is called the {\em degree} of the PB constraint.
We call the set of variables that are present in a PB constraint $C$ the {\em support} of $C$, denoted as $\support{C}$.

A {\em full assignment} $\assignment : X \mapsto \{0,1\}$ assigns a truth value to each variable in $X$. 
We denote the truth value that $\assignment$ assigns to a variable $x \in X$ by $\assignment(x)$.
We can {\em evaluate} a PB constraint to see if it is {\em satisfied} by an assignment $\assignment$.
This is done by substituting the literals in $C$ for their values in $\assignment$: $\assignment(C) := C\left(\left\{x \mapsto \assignment(x) \mid x \in \support{C} \right\}\right)$, evaluating the left-hand side and the right-hand side, and checking that they yield integers that satisfy the {\em relational operator} $\bowtie$.

A {\em PB formula} $\pbformula(X)$ is a set of PB constraints on Boolean variables $X$. 
We say that $\pbformula(X)$ is {\em satisfiable} iff there exists an assignment $\assignment$ that satisfies each constraint $C \in \pbformula$.
If no such $\assignment$ exists, we say that $\pbformula$ is {\em unsatisfiable}.

The process of showing that a PB formula is unsatisfiable, is called {\em refuting}, and a proof of unsatisfiability is called a {\em refutation}.
A solver that solves the decision problem of determining whether a PB formula is satisfiable, is called a {\em PB solver}.

\subsection{Verifying Unsatisfiability}
\label{subsec:verifying-unsat}

Given a candidate solution $\assignment$ to a PB formula $\pbformula(X)$, it is straightforward to verify that $\assignment$ is indeed a solution to $\pbformula(X)$.
All that it requires, is to substitute every variable in $\pbformula(X)$ by the truth value that is assigned to it by $\assignment$, and check that each constraint $C \in \pbformula$ is satisfied by $\assignment$.
However, it is less trivial to verify that a PB formula $\pbformula(X)$ is {\em unsatisfiable} and thus has no solutions.
If a PB solver returns ``unsatisfiable'', our confidence in the correctness of that conclusion is therefore based on the solver's reputation, \eg{}, on who developed the solver or on how thoroughly it was tested.

Since PB solvers can be complex tools, they may contain bugs or even implementation mistakes due to conceptual errors.
Hence, for applications where the stakes are high, a solver's reputation may not be enough to ensure the social acceptability of the refutations that it produces.

This does not only hold for PB solvers, but also for other solvers, like Boolean satisfiability (SAT) solvers.
For the last two decades, the Boolean Satisfiability community has therefore been investigating the practices of {\em proof logging} and {\em proof verification}~\cite{GN03,ZM03,Van08,HHW13a,HHW13b,WHH14}.
Since 2013, it has become common practice for SAT solvers to produce a {\em certificate of unsatisfiability} or {\em proof log} along with their unsatisfiability results.
This (potentially very large) file contains a log of the reasoning steps that justify the solver's conclusion.
SAT solvers participating in Boolean Satisfiability competitions~\cite{FHI+21} are required to produce certificates of unsatisfiability in the standardised \drat{} (``deletion resolution asymmetric tautology'') format~\cite{HHW13a,HHW13b,WHH14}.

Recently, the PB solving community followed the example set by the SAT solving community of requiring solvers to provide refutation proofs.
Modern PB solvers, like \roundingsat{}~\cite{EN18,Dev20,DGN21,DGD+21}, produce a certificate of unsatisfiability when given an unsatisfiable PB formula $\pbformula$.
This certificate consists of a list of cutting-planes inferences~\cite{Gom58}, where new PB constraints are derived from existing PB constraints, until the constraint $0 \geq 1$ (a contradiction) is derived.
A {\em proof verifier}, like \veripb{}~\cite{GMM+20,GMNO22,BGMN22,EGMN20,GMN20,GN21}, then takes as input the original PB formula $\pbformula$ and the certificate \texttt{cert}, and checks that all the inferences in the certificate are correct.
If the verifier concludes that they are, it has successfully verified the correctness of the proof of unsatisfiability.

PB researchers have not only developed technology for certifying the unsatisfiability of sets of PB constraints~\cite{VEG+18,GMM+20}, but also worked on techniques for certifying the translation of PB constraints into Boolean formulae~\cite{GMNO22}, and on certifying the optimality of solutions of PB optimisation problems~\cite{BGMN22}.
In this work, however, we focus on refutation proofs of sets of linear PB constraints.

Note that verifiers do not provide proof that the solver operates correctly on all inputs, or that it is bug-free.
All that a verifier can do, is check that for a given instance of an unsatisfiable formula, the solver made no mistakes in refuting the satisfiability of that instance.
Hence, if the solver correctly identifies an unsatisfiable instance as such, but does so using flawed reasoning, the verifier still rejects the certificate.
This makes the verification of refutation certificates a useful tool in detecting (implementation) bugs in solvers~\cite{FHR22}.

Crucially, verifiers are much simpler tools than SAT solvers or PB solvers. 
Hence, they are much easier to debug and check manually, which gives us confidence in their correctness.
Moreover, because verifiers are relatively simple (and therefore small) programs, it is possible to {\em formally verify} their correctness~\cite{Lam20}, using formally defined languages like proof assistants \acltwo{}, \coq{}, \holfour{}, or \isabelle{}.\footnote{Available at  \url{https://www.cs.utexas.edu/users/moore/acl2}, \url{https://coq.inria.fr}, \url{https://hol-theorem-prover.org/}, and \url{https://isabelle.in.tum.de}.}

\subsection{Proving Mathematical Conjectures with Unsatisfiability Proofs}
\label{subsec:proving-mathematical-conjectures}

Machine-verifiable proof technology gave rise to a line of research in which mathematical conjectures were proved using SAT solvers.
Notable examples include a proof that Schur Number Five is $S(5) = 160$ (a centuries-old open problem)~\cite{Heu18}, and a $200$ terabyte proof of the solution to the related Boolean Pythagorean triples problem~\cite{Heu17}, which took $40\:000$ CPU hours of computation to generate, including verification.

The recipe for proving a conjecture is the following. 
First, encode the {\em negation} of the conjecture.
This encoding can be a propositional formula, for example in {\em conjunctive normal form (CNF)}, or a PB formula, or even a sentence in a more expressive constraint specification language.
Then, let an appropriate solver refute the satisfiability of the encoded negated conjecture.
Finally, let a verification tool validate the unsatisfiability certificate.

In order for these machine-generated proofs to be socially accepted, we must convince ourselves of the following:
\begin{enumerate}
	\item The encoding of the conjecture is correct.
	\item The implementation of that encoding is correct.
	\item The verifier is correct.
\end{enumerate}

The first item on this list is impossible to formally guarantee.
The best we can do, is to make the encoding as simple and as straightforward as possible, so we can convince the reader that it is correct.
Similarly, the second item relies on implementing the encoding in a succinct and clear way, so the reader can convince themselves of the correctness of the implementation.
Finally, we must be convinced that the verifier made no mistakes in generating the proof, which we discussed at the end of \cref{subsec:verifying-unsat}.

\subsection{Cutting-Planes Proofs of Unsatisfiability}
\label{subsec:cutting-planes}

Proofs of unsatisfiability of sets of PB constraints all use a set of {\em inference rules} to derive new constraints from existing constraints, until the contradictory constraint $0 \geq 1$ is derived.
The first rules are simply definitions and axioms:
\begin{align}
	&\inferrule{}{x \geq 0}\qquad &\text{lower bound} \label{axiom:lb}\\
	&\inferrule{}{-x \geq -1}\qquad &\text{upper bound} \label{axiom:ub}\\
	&\inferrule{}{\overline{x} = 1 - x}\qquad &\text{negation} \label{axiom:negation} \\
	&\inferrule{}{x \cdot x = x}\qquad &\text{idempotence} \label{axiom:idempotence}
\end{align}

The proof system for PB constraints is the {\em cutting-planes proof system}~\cite{Gom58}. 
It has the following three rules to derive new PB constraints (under the horizontal line) from existing constraints and constants (above the line):
\begin{align}
	&\inferrule{\sum_i a_i \ell_i \geq b \qquad \sum_i c_i \ell_i \geq d}{\sum_i (a_i + c_i)  \ell_i \geq b + d} \qquad &\text{addition} \label{rule:addition}\\
    &\inferrule{\sum_i a_i \ell_i \geq b}{\sum_i \alpha  a_i  \ell_i \geq \alpha  b}, \:\forall \alpha \in \mathbb{N}^+ \qquad &\text{multiplication} \label{rule:multiplication}\\
	&\inferrule{\sum_i a_i \ell_i \geq b}{\sum_i \ceil*{\frac{a_i}{\alpha}}  \ell_i \geq \ceil*{\frac{b}{\alpha}}},\:\forall \alpha > 0 \qquad &\text{division} \label{rule:division}
\end{align}
Anything above a horizontal line is called a {\em premise}, and everything below it a {\em conclusion}.

Other well-known rules include the {\em saturation rule}, the {\em (partial) weakening rule}, the {\em strengthening rule} and the {\em redundancy rule}. 
These rules can all be replaced by a series of applications of the four axioms and three rules described above, and hence function as syntactic sugar in PB proofs.
Hence, a detailed discussion of these rules is outside the scope of this work, and we refer the reader to the literature for more details~\cite{RM21}.

In practice, a proof of unsatisfiability has to be written and read by a computer.
This is done by indexing all the constraints in the original problem and the derived constraints.
Whenever a derivation rule is applied to one or more constraint(s), this is logged in the proof by creating a new line that labels the rule that is being applied, and references the index/indices of the constraint(s) it is applied on.
This helps the verifier to check if all these rules are applied correctly.

We now provide an example of an unsatisfiable set of PB constraints, a cutting-planes proof of the unsatisfiability of this set of constraints, and a machine-readable version of that proof.

\begin{example}
\label{ex:cutting-planes-refutation}
Consider the following PB formula, where we have labelled the constraints as lemmas \texttt{L1}--\texttt{L7}:
\begin{gather}
\begin{aligned}
    &\texttt{L1}:\quad x_1 + x_2 + x_3 \geq 1 \\
    &\texttt{L2}:\quad x_2 + 2 x_3 + 3 x_4 \geq 3 \\
    &\texttt{L3}:\quad 2 x_1 + x_3 + 2 x_4 \geq 2 \\
    &\texttt{L4}:\quad x_1 + x_2 + x_4 \geq 1 \\
    &\texttt{L5}:\quad x_1 + x_3 \geq 1 \\
    &\texttt{L6}:\quad x_2 + x_4 \geq 1 \\
    &\texttt{L7}:\quad \overline{x}_1 + \overline{x}_2 + \overline{x}_3 + \overline{x}_4 \geq 3
\end{aligned}
\label{eq:unsat-pb}
\end{gather}
\end{example}

In the following, we will derive new facts and number them as `inferences' \texttt{i8}--\texttt{i13}.
Note that the negation and upper bound axioms imply $\overline{x} \geq 0$ (to check, simply substitute $x = 1 - \overline{x}$ into the upper bound inequality and simplify).
We use the fact `$\overline{x} \geq 0$' in the derivation of the new facts, for the sake of brevity.

\begin{example}
\begin{figure}
    \centering
    \begin{subfigure}[t]{0.525\textwidth}
        \begin{align}
            &\inferrule{\texttt{L3} \quad \overline{x}_3 \geq 0}{\texttt{i8a}:\quad 2 x_1 + 2 x_4 \geq 2} \quad \text{(addition and negation)} \label{eq:deriv8a} \\
            &\inferrule{\texttt{i8a}}{\texttt{i8b}:\quad x_1 + x_4 \geq 1} \quad \text{(division by $2$)} \label{eq:deriv8b}\\
            &\inferrule{\texttt{i8b} \quad \texttt{L7}}{\texttt{i8}:\quad \overline{x}_2 +  \overline{x}_3 \geq 2} \quad \text{(addition)} \label{eq:deriv8}\\
            &\inferrule{\texttt{i8} \quad x_3 \geq 0}{\texttt{i9}:\quad \overline{x}_2 \geq 1} \quad \text{(addition)} \label{eq:deriv9} \\
            &\inferrule{\texttt{i8} \quad x_2 \geq 0}{\texttt{i10}:\quad \overline{x}_3 \geq 1} \quad \text{(addition)}\\
            &\inferrule{\texttt{L6} \quad \texttt{i9}}{\texttt{i11}:\quad x_4 \geq 1} \quad \text{(addition)}\\
            &\inferrule{\texttt{L5} \quad \texttt{i10}}{\texttt{i12}:\quad x_1 \geq 1} \quad \text{(addition)}\\
            &\inferrule{\texttt{L7} \quad \texttt{i12} \quad x_2 \geq 0 \quad x_3 \geq 0 \quad \texttt{i11}}{\texttt{i13}:\quad 0 \geq 1} \quad \text{(addition)}
        \end{align}
    \caption{Cutting-planes proof that the set of PB constraints in \cref{ex:cutting-planes-refutation} is unsatisfiable. The last derived fact is a contradiction.} 
    \label{subfig:cutting-planes-proof}
    \end{subfigure}
    \hfill
    \begin{subfigure}[t]{0.375\textwidth}
        \begin{lstlisting}[firstnumber=0]
pseudo-Boolean proof version 1.0
u >= 0 ;
l 1
l 2
l 3
l 4
l 5
l 6
l 7
p 8 4 ~x3 + 2 d + 0
p 9 x3 + 0
p 9 x2 + 0
p 7 10 + 0
p 6 11 + 0
p 8 13 + x2 + x3 + 12 + 0
c 14 0
        \end{lstlisting}
    \caption{\roundingsat{}'s proof of the unsatisfiability of \cref{eq:unsat-pb}. The line numbers represent the constraint IDs, which are referenced in lines 9--15. The last line states contradiction.} 
    \label{subfig:roundsat-proof}
    \end{subfigure}
    \caption{A proof that the set of PB constraints in \cref{eq:unsat-pb} is unsatisfiable.}
    \label{fig:enter-label}
\end{figure}

Consider the cutting-planes proof in \cref{subfig:cutting-planes-proof}, and the machine-readable version of the same proof in \cref{subfig:roundsat-proof}, produced by \roundingsat{}.
The line numbers in the latter represent the constraint IDs.
The line `\verb|u >= 0 ;|' has constraint ID \verb|1|, and indicates that all literals must have a value of at least $0$.
Hence, axiom \verb|L1| has constraint ID \verb|2|, and corresponds to the line `\verb|l 1|' in \cref{subfig:roundsat-proof}, where the `\verb|1|' in that line refers to the line number of that axiom in the input file.

Now consider the derivation steps \cref{eq:deriv8a,eq:deriv8b,eq:deriv8} in \cref{subfig:cutting-planes-proof}, which correspond to the derivation line number \verb|9| in \cref{subfig:roundsat-proof}.
For clarity, we have split it up in three derivation steps.
First, \cref{eq:deriv8a} takes lemma \verb|L3| and the negation of $x_3$ and performs the following addition: 
\begin{equation}
\begin{array}{*{10}{c}}
  ~ & 2 x_1 & + & ~ & ~ & x_3     & + & 2 x_4 & \geq & 3   \\
  + \qquad & ~     & ~ & ~ & ~ & 1-x_3 & ~ & ~     & \geq & 0   \\ \hline
  = \qquad & 2 x_1 & + & ~ & ~ & 1       & + & 2 x_4 & \geq & 3,   \\
\end{array}
\end{equation}
which simplifies to $2 x_1 + 2 x_4 \geq 2$, the derived constraint \verb|i8a|. 
The inference in \cref{eq:deriv9} then divides this constraint by $2$, according to the division rule in \cref{rule:division}, to obtain $x_1 + x_4 \geq 1$, which is the derived constraint \verb|i8b|.
The final derivation made to obtain the constraint with constraint ID \verb|9| (labelled \verb|i8| in \cref{subfig:cutting-planes-proof}), is the addition of \verb|i8b| and \verb|L7|.
We invite the reader to verify that the remaining derivation steps are indeed correct.
\end{example}

In the above derivation, we have merged addition, negation and simplification steps to keep the proof of manageable size.
In the final step, we have merged four addition steps into one.
Deriving $0 \geq 1$ in the final step proves that there exists no solution $\assignment : \{x_1, x_2, x_3, x_4\} \mapsto \{0,1\}$ that satisfies all PB constraints in \cref{eq:unsat-pb}.

\begin{example}
The above proof is logged as follows in the format accepted by \veripb{}:
The first seven facts are simply marked by their line numbers in the input file (\verb|l 1| to \verb|l 7|).
The lines of the proof are all in reverse Polish notation, and end in a `\verb|0|'.
For example, the line \verb|p 9 x3 + 0| corresponds to the derivation of \texttt{i9} in the example above.
In the computer-readable proof, more steps can be merged into one line.
For example, the first inference steps (which derive \texttt{i8a}, \texttt{i8b} and \texttt{i8} in the example above) are merged in the computer-readable proof to form line \texttt{p 8 4 ~x3 + 2 d + 0} in reversed Polish notation. 
\end{example}

A detailed discussion of the proof format is outside the scope of this work, and a description of this format can be found at \href{https://github.com/StephanGocht/VeriPB}{github.com/StephanGocht/VeriPB}.

\section{Machine-Generated and Machine-Verified Lower Bound of MICS of SBG}
\label{sec:pb-proof}

We now describe how we use machine-verifiable proofs to prove the following theorem:
\begin{theorem}
	The MICS of SBG is at least ten.
    \label{theorem:at-least-10}
\end{theorem}

First, we show how to encode the GCS problem as a PB formula.
We then describe the steps we took to generate and verify a proof of this theorem.

\subsection{Pseudo-Boolean Encoding}
\label{subsec:pb-encoding}

Recall the Identifying Code Set decision problem as described in \cref{def:ICS-decision-problem}.
Using the Graph Coloring with Seepage (GCS) terminology as described in \cref{subsec:GCS}, we now describe a PB encoding of the ICS decision problem.

In the following discussion, we will use $dist(u,v)$ to denote the distance between two nodes in a graph, \ie{}, the number of edges along the shortest path from $u$ to $v$ (with $dist(v,v) := 0$).
Additionally, we define $N_{\leq 2}(v) := \{u \in V \mid dist(u,v) \leq 2\}$ to be the {\em closed $2$-neighborhood} of a node $v \in V$.
We define $DS(u,v) := N^+(u) \triangle N^+(v)$ to be the {\em distinguishing set} of nodes $u,v \in V$, \ie{}, the symmetric difference of $N^+(u)$ and $N^+(v)$.

We encode the decision version of the GCS/ICS decision problem into a set of PB constraints as follows.
First, we define the Boolean variables $X = \{ x_v \mid v \in V\}$ to be the injection variables, such that $x_v = 1$ indicates that node $v$ is injected with a color and $x_v = 0$ that it is not injected with any color. 
The first set of constraints in $\pbformula(X)$ expresses that each node has to receive a color (either through injection or seepage).
To ensure that a node $v \in V$ is colored, it is enough to ensure that color is injected into $v$ itself, or into one of the nodes that are adjacent to $v$:
\begin{equation}
	\pbformula_{\text{alo}}(X) := \left\{ \sum_{u \in N^+(v)} x_u \geq 1 \mid v \in V\right\}.
	\label{eq:pb-alo}
\end{equation}
The next set of constraints that we need to define, express that all signatures must be unique.
For the signatures of two distinct nodes $u,v \in V$ to coincide, they must have at least one color in common.
The only way in which $u$ and $v$ can share a color, is if they are within a  distance of at most $2$ of each other.
Hence, we only need to explicitly encode that the signatures of nodes $u, v \in V$ are different if they are in each other's closed $2$-neighbourhood.
Furthermore, the only way in which the coloring of $u$ can be different from the coloring of $v$, is if at least one node in their distinguishing set is injected with color.
With these observations, we define the following set of uniqueness constraints:
\begin{equation}
    \pbformula_{\text{unique}} (X) := \left\{ \sum_{w \in DS(u,v)} x_w \geq 1 \mid v \in V, u \in N_{\leq 2}(v) \right\}
\end{equation}
Finally, we add a constraint that says that, for the SBG, there exists a solution with at most nine sensors:
\begin{equation}
	C_{\text{budget}}(X) :=  \sum_{v \in V} x_v \leq 9.\footnote{
        In practice, PB solvers only accept relational operators `$\geq$' and `$=$', so we must convert this constraint into the following form for it to serve as input to a PB solver: $C_{\text{budget,norm}}(X) :=  -\sum_{v \in V} x_v \geq -9$.
    }
    \label{eq:budget}
\end{equation}
With this, we obtain the following PB formula:
\begin{equation}
	\pbformula_{\abs{\text{MICS}} \leq 9}(X) := \constraint{\text{alo}} \cup \constraint{\text{unique}} \cup \left\{C_{\text{budget}}\right\}.
	\label{eq:pb}
\end{equation}
If this formula is unsatisfiable for the SBG, we know that the cardinality of the MICS for the SBG must be greater than nine.
Since we have proved in \cref{sec:combinatorial-lower-bound} that the cardinality of the MICS of the SBG is at least nine, and proved in \cref{sec:upper-bound} that it is at most ten, by proving that it is greater than nine, we prove that the cardinality of the MICS of the SBG is indeed ten.

\subsection{Proving that ten is a Strict Lower Bound for the SBG}

We implemented a script that takes a network, in this case the edge list of the SBG, and returns \cref{eq:pb} in OPB format~\cite{RM12}.
This script was implement in {\sf Python 3.12.1}, using the {\sf Networkx 3.1} library to compute the neighbourhood functions.
The result was a PB formula consisting of $273$ PB constraints.

Then, we used \roundingsat{}~\cite{EN18,Dev20,DGN21,DGD+21}\footnote{Available at \href{https://gitlab.com/MIAOresearch/software/roundingsat}{gitlab.com/MIAOresearch/software/roundingsat}. We used commit $c548e1098a81d1f57dfc31560208034253d174c1$, retrieved on $24$ January $2024$.} to generate a certificate of unsatisfiability of the formula.
The proof is $3884$ lines long, and has a size of $359$ KiB.
It is available in our project repository, along with all other scripts and related files: \href{https://github.com/latower/SBG-bounds}{github.com/latower/SBG-bounds}.
We then used \veripbvtwo{}~\cite{GMM+20,GMNO22,BGMN22,EGMN20,GMN20,GN21,BMM+23}\footnote{Available at \href{https://gitlab.com/MIAOresearch/software/VeriPB}{gitlab.com/MIAOresearch/software/VeriPB}. We used commit $9a1f7588339877471b7b8cf81a3b9702ee494920$, retrieved on $24$ January $2024$.} to verify that all inferences made in the certificate of unsatisfiability are indeed valid, and \veripbvtwo{} confirmed that they are. 
Hence, {\em there exists no MICS of cardinality at most nine for the SBG, so the cardinality of a MICS of the SBG must be at least ten}.

Recall the three conditions for accepting a machine-generated proof that we mentioned in \cref{subsec:proving-mathematical-conjectures}. 
We described our encoding of the statement that the MICS of the SBG is at least ten (\cref{theorem:at-least-10}) in \cref{subsec:pb-encoding}, and invite the reader to convince themselves of its correctness, thus crossing off the first condition.
We provide the implementation of that encoding in the following repository: \href{https://github.com/latower/SBG-bounds}{github.com/latower/SBG-bounds}.
The implementation consists of three {\sf Python} scripts: one to parse an edge list that specifies the SBG, one to convert the graph and a given number of satellites $b$ into a set of PB constraints in OPB format, and one wrapper script to call the encoder.
Again, we invite the reader to study that code and convince themselves of its correctness, which covers the second condition.
Finally, to convince the reader of the correctness of the proof checker, we mention that \veripbvtwo{}'s kernel proof checker, \cakepb{}, has been formally verified using interactive theorem prover \holfour{}~\cite{SN08,BMM+23}. 
Hence, we can trust that it is bug-free, thus convincing us of point (3) that was mentioned in \cref{subsec:proving-mathematical-conjectures}. 

To confirm that there exists at least one solution to minimal identifying code set problem for the SBG, we also encoded $\pbformula_{\abs{\text{MICS}} \leq 10}(X)$, in which \cref{eq:budget} has a degree of ten instead of nine.
Encoding $\pbformula_{\abs{\text{MICS}} \leq 10}(X)$ as an OPB file and solving it with \roundingsat{} showed that, indeed, $\pbformula_{\abs{\text{MICS}} \leq 10}(X)$ is satisfiable.
This confirms that there exists an MICS with cardinality ten for the SBG.

\subsection{Counting the number of solutions}

In \cref{sec:upper-bound}, we showed analytically that there exist at least $26$ distinct identifying code sets (ICSes) with cardinality $10$ for the soccer ball graph (SBG), proving that \cref{thrm:at-least-26} holds.
Using the PB formula $\pbformula_{\abs{\text{MICS}}\geq 10}(X)$, we can prove a stronger statement:
\begin{theorem}
    \label{thrm:exactly-26}
    There exist {\em exactly} $26$ distinct size-$10$ ICSes of the SBG.
\end{theorem}

To prove that \cref{thrm:exactly-26} holds, we wrote a script that does the following.
It takes $\pbformula_{\abs{\text{MICS}}\geq 10}(X)$ and uses \roundingsat{} to find a solution $\sigma : X \mapsto \{0,1\}$.
Suppose that $X_{1} \subseteq X$ are all the variables that are mapped to $1$ by $\sigma$, and $X_0 \subseteq X$ are all the variables that are mapped to $0$ by $\sigma$, such that $X_1 := \{ x \in X \mid \sigma(x) = 1\}$, $X_0 := \{ x \in X \mid \sigma(x) = 0\}$, $X = X_1 \cup X_0$ and $X_1 \cap X_0 = \varnothing$.
Now, the script generates the following constraint:
\begin{equation}
    \constraint{\neg \sigma} := \sum_{x \in X_1} (1-x) + \sum_{x \in X_0} x \geq 1,
\end{equation}
which expresses that at least one variable $x \in X$ must take a different value than what it was assigned by solution $\sigma$. 
We call this constraint a {\em blocking constraint} for $\sigma$, because it cannot be satisfied by $\sigma$.
This blocking constraint is rewritten to the following constraint, so it is in the right input format for our PB solver:
\begin{equation}
    \constraint{\neg \sigma} := -\sum_{x \in X_1} x + \sum_{x \in X_0} x \geq 1 - \abs{X_1}.
\end{equation}
Then, the script generates a new formula 
\begin{equation}
    \pbformula_{\abs{\text{MICS}}\geq 10}^\prime(X) := \pbformula_{\abs{\text{MICS}}\geq 10}(X) \cup \{\constraint{\neg \sigma}\},
\end{equation}
and uses \roundingsat{} to find a solution $\sigma^\prime : X \mapsto \{0,1\}$, which is then turned into a blocking constraint, and added to the formula.
This process continues until the PB formula becomes unsatisfiable.

The script then counts all the solutions, checks that they are all distinct, checks that each found solution is indeed a solution to the original PB formula $\pbformula_{\abs{\text{MICS}}\geq 10}(X)$, and uses \veripbvtwo{} to verify that the refutation certificate of the final, unsatisfiable formula is correct.
All scripts and results can be found in our repository: \href{https://github.com/latower/SBG-bounds}{\texttt github.com/latower/SBG-bounds}.
\FloatBarrier
\section{Conclusion}
\label{sec:conclusion}

We have studied an event monitoring problem with satellites as sensors and a soccer ball as a model of the planet Earth. 
Here, each patch of the soccer ball models a region of the globe that can be monitored by one satellite for big events. 
The assumption is that the effects of a big event taking place in one region, spills out to the adjacent regions.
Hence, we do not need to employ one satellite per region to monitor the entire earth for big events.
The problem is to find a cardinality-minimal set of regions for which we must deploy a satellite in order to monitor the entire earth, and uniquely identify in which region a big event takes place.
We modelled this problem as a {\em minimum identifying code set (MICS)} problem.

In this work, we provided human-oriented, analytical proofs that the MICS is at least $9$ and at most $10$ for the soccer ball graph that models satellite deployment problem.
Then, we provided a machine-oriented proof that the MICS of the soccer ball is $10$, meaning that there exists no way to employ $9$ or fewer satellites in the soccer ball model and be able to monitor all regions and uniquely identify the region in which a big event takes place.
To this end, we modelled the problem as a set of Pseudo-Boolean constraints, in which we fix the number of satellites to at most $9$.
We then used a Pseudo-Boolean solver to generate a machine-readable proof that such a set of constraints is unsatisfiable, whereas the set of constraints {\em is} satisfiable if we fix the number of deployed satellites to at most $10$.

We also provided a human-oriented, analytical proof that the SBG has at least $26$ distinct identifying code sets of cardinality $10$.
We then used a machine-oriented approach to prove that SBG has in fact {\em exactly} $26$ distinct identifying code sets of cardinality $10$.

In our future work, we plan to study regions with irregular shapes, thereby relaxing the hexagonal and pentagonal region constraints.
We also plan to use machine-generated and machine-verified proofs to other (quantitative) results on sensor placement problems.

\section*{Acknowledgements}
This work is supported by National Research Foundation Singapore, under its NRF Fellowship Programme under Grant No. NRF-NRFFAI1-2019-0004, by the Ministry of Education Singapore Tier 2 Grant No. MOE-T2EP20121-0011, and by the Ministry of Education Singapore Tier 1 Grant No. R-252-000-B59-114.

\printbibliography

\end{document}